\documentclass{article}

\usepackage{times}
\usepackage{epsfig}
\usepackage{graphicx}
\usepackage{amsmath}
\usepackage{amssymb}

\usepackage{microtype}
\usepackage{booktabs} 

\usepackage[utf8]{inputenc} 
\usepackage[T1]{fontenc}    
\usepackage{url}            
\usepackage{amsfonts}       
\usepackage{nicefrac}       
\usepackage[inline]{enumitem}

\usepackage{wrapfig, blindtext}
\usepackage{epstopdf}
\usepackage{booktabs} 
\usepackage{extarrows}

\usepackage{amsmath}   
\usepackage{amssymb}
\usepackage{wasysym}
\usepackage{multirow}

\usepackage{amsthm}
\usepackage{bbm}
\usepackage{bm}
\usepackage{enumerate}
\usepackage{footmisc}
\usepackage{algorithm}
\usepackage{algorithmic}
\usepackage{setspace}
\usepackage{caption}
\usepackage{makecell}

\usepackage{natbib}
\usepackage[utf8]{inputenc} 
\usepackage[T1]{fontenc}    
\usepackage{url}            
\usepackage{booktabs}       
\usepackage{graphicx}
\usepackage{wrapfig,lipsum,booktabs}
\usepackage{subcaption}
\usepackage{amsmath}
\usepackage{graphicx}

\usepackage{amsfonts}
\usepackage{mathtools}
\usepackage{framed}
\usepackage{microtype}

\usepackage{siunitx}
\usepackage{xr-hyper}
\usepackage{xr}
\usepackage{xr}
\usepackage{wrapfig}
\usepackage{amsmath,amsthm}
\usepackage{amssymb}
\usepackage{epstopdf}
\usepackage{mathtools}
\usepackage[dvipsnames]{xcolor}

\usepackage[pagebackref=true,breaklinks=true,colorlinks,bookmarks=false]{hyperref}

\definecolor{mygray}{gray}{0.5}
\newcommand{\one}[1]{\mathbbm{1}_{[#1]}}

\newcolumntype{K}[1]{>{\centering\arraybackslash}p{#1}}

\newcommand{\dacl}{DACL }
\newcommand{\daclp}{DACL+ }
\newcommand{\daclns}{DACL}

\usepackage{Definitions}
\newcommand{\cont}{{\normalfont{\text{ctr}}}} 
\newcommand{\gauss}{{\normalfont{\text{gauss}}}} 
\newcommand{\mix}{{\normalfont{\text{mix}}}} 
\newcommand{\simi}{{\normalfont{\text{sim}}}} 
\newcommand{\class}{{\normalfont{\text{cf}}}} 

\newcommand{\quotes}[1]{``#1''}

\usepackage{cleveref}

\usepackage[accepted]{icml2021}

\icmltitlerunning{Towards Domain-Agnostic Contrastive Learning}

\begin{document}

\twocolumn[
\icmltitle{Towards Domain-Agnostic Contrastive Learning}




\begin{icmlauthorlist}
\icmlauthor{Vikas Verma}{goo,aalto}
\icmlauthor{Minh-Thang Luong}{goo}
\icmlauthor{Kenji Kawaguchi}{har}
\icmlauthor{Hieu Pham}{goo}
\icmlauthor{Quoc V. Le}{goo}

\end{icmlauthorlist}

\icmlaffiliation{aalto}{Aalto University, Finland.}
\icmlaffiliation{goo}{Google Research, Brain Team.}
\icmlaffiliation{har}{Harvard University}

\icmlcorrespondingauthor{Vikas Verma}{vikas.verma@aalto.fi}
\icmlcorrespondingauthor{Minh-Thang Luong}{thangluong@google.com}
\icmlcorrespondingauthor{Kenji Kawaguchi}{kkawaguchi@fas.harvard.edu}
\icmlcorrespondingauthor{Hieu Pham}{hyhieu@google.com}
\icmlcorrespondingauthor{Quoc V. Le}{qvl@google.com}

\icmlkeywords{Machine Learning, ICML}

\vskip 0.3in
]



\printAffiliationsAndNotice{\icmlEqualContribution} 

\begin{abstract}
Despite recent successes, most contrastive self-supervised learning methods  are domain-specific, relying heavily on data augmentation techniques that require knowledge about a particular domain, such as image cropping and rotation.
To overcome such limitation, we propose a domain-agnostic approach to contrastive learning, named {\it DACL}, that is applicable to problems where 
domain-specific data augmentations are not readily available. 
Key to our approach is the use of {\it Mixup noise} to create similar and dissimilar examples by mixing data samples differently either at the input or hidden-state levels. We theoretically analyze our method and show advantages over the Gaussian-noise based contrastive learning approach.
To demonstrate the effectiveness of DACL, we conduct experiments across various domains such as tabular data, images, and graphs. 
Our results show that DACL not only outperforms other domain-agnostic noising methods, such as Gaussian-noise, but also combines well with domain-specific methods, such as SimCLR, to improve self-supervised visual representation learning.
\end{abstract}

\section{Introduction}

One of the core objectives of deep learning is to discover useful representations from the raw input signals without explicit labels provided by human annotators. Recently, self-supervised  learning methods have emerged as one of the most promising classes of methods to accomplish this objective with strong performances across various domains such as computer vision ~\citep{Oord2018RepresentationLW,He2019MomentumCF,Chen2020ASF, byol}, natural language processing ~\citep{dai2015semi,howard2018universal,peters2018deep,radford2019language, Clark2020ELECTRA}, and speech recognition ~\citep{schneider2019wav2vec,baevski2020wav2vec}. These self-supervised methods learn useful representations without explicit annotations by reformulating the unsupervised representation learning problem into a supervised learning problem. This reformulation is done by defining a pretext task. The pretext tasks defined in these methods are based on certain domain-specific \textit{regularities} and would generally differ from domain to domain (more discussion about this is in the related work, Section \ref{sec:related_work}). 

%

\begin{figure}[ht]
\begin{center}
\includegraphics[width=0.85\linewidth]{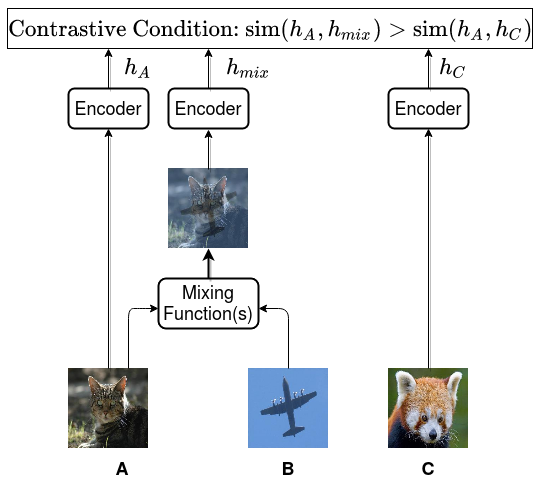}
\end{center}
\vspace{-0.4cm}
\caption{For a given sample A, we create a positive sample by mixing it with another random sample B. The mixing function can be either of the form of Equation \ref{eq:mcl} (Linear-Mixup), \ref{eq:mix_geo} (Geometric-Mixup) or \ref{eq:mix_fm} (Binary-Mixup), and the mixing coefficient is chosen in such a way that the mixed sample is closer to A than B. Using another randomly chosen sample C, the contrastive learning formulation tries to satisfy the condition $\simi(\bm{h_A}, \bm{h}_{mix}) > \simi(\bm{h_A}, \bm{h}_C)$, where $\simi$ is a measure of similarity between two vectors.}
\vspace{-0.4cm}
\label{fig:dacl}
\end{figure}

Among various pretext tasks defined for self-supervised learning, contrastive learning, e.g.  \cite{chopra,hadsell,Oord2018RepresentationLW, cpc_v2, He2019MomentumCF,Chen2020ASF,NEURIPS2020_4c2e5eaa,NEURIPS2020_9523147e, wang2020hypersphere}, is perhaps the most popular approach that learns to distinguish semantically similar examples over dissimilar ones. Despite its general applicability, contrastive learning requires a way, often by means of data augmentations, to create semantically similar and dissimilar examples in the domain of interest for it to work. For example, in computer vision, semantically similar samples can be constructed using semantic-preserving augmentation techniques such as flipping, rotating, jittering, and cropping.
These semantic-preserving augmentations, however, require domain-specific knowledge and may not be readily available for other modalities such as graph or tabular data.

How to create semantically similar and dissimilar samples for new domains remains an open problem. 
As a simplest solution, one may add a sufficiently small random noise (such as Gaussian-noise) to a given sample to construct examples that are similar to it.
Although simple, such augmentation strategies do not exploit the underlying structure of the data manifold. 
In this work, we propose DACL, which stands for \underline{D}omain-\underline{A}gnostic \underline{C}ontrastive \underline{L}earning, an approach that utilizes 
\textit{Mixup-noise} to create similar and dissimilar examples by mixing data samples differently either at the input or hidden-state levels. A simple diagrammatic depiction of how to apply \dacl in the input space is given in Figure \ref{fig:dacl}. 
Our experiments demonstrate the effectiveness of DACL across various domains, ranging from tabular data, to images and graphs; whereas, our theoretical analysis sheds light on why Mixup-noise works better than Gaussian-noise.

In summary, the contributions of this work are as follows:
\begin{itemize}
    \item We propose Mixup-noise as a way of constructing positive and negative samples for  contrastive learning and conduct theoretical analysis to show that Mixup-noise has better generalization bounds than Gaussian-noise.
    \item We show that using other forms of data-dependent noise (geometric-mixup, binary-mixup) can further improve the performance of \daclns.
    
    \item We extend \dacl to domains where data has a non-fixed topology (for example, graphs) by applying Mixup-noise in the hidden states.
   
    
    \item We demonstrate that Mixup-noise based data augmentation is complementary to other image-specific augmentations for contrastive learning, resulting in improvements over SimCLR baseline for CIFAR10, CIFAR100 and ImageNet datasets.
\end{itemize}

\section{Contrastive Learning : Problem Definition}
Contrastive learning can be formally defined using the notions of \quotes{anchor}, \quotes{positive} and \quotes{negative} samples. Here, positive and negative samples refer to samples that are semantically similar and dissimilar to anchor samples. Suppose we have an encoding function $h: \bm{x} \mapsto \bm{h}$, an anchor sample $\bm{x}$ and its corresponding positive and negative samples, $\bm{x}^{+}$ and $\bm{x}^{-}$. The objective of contrastive learning is to bring the anchor and the positive sample closer in the embedding space than the anchor and the negative sample. Formally, contrastive learning seeks to satisfy the following condition, where $\simi$ is a measure of similarity between two vectors:
\begin{equation}
    \simi(\bm{h}, \bm{h}^{+}) > \simi(\bm{h}, \bm{h}^{-})
\end{equation}
While the above objective can be reformulated in various ways, including max-margin contrastive loss in \cite{hadsell}, triplet loss in \cite{triplet_loss}, and maximizing a metric of local aggregation \citep{Zhuang2019LocalAF},  in this work we consider InfoNCE loss because of its adaptation in multiple current state-of-the-art methods \citep{sohn,Oord2018RepresentationLW,He2019MomentumCF,Chen2020ASF,Wu_2018_CVPR}. Let us suppose that $\{\bm{x}_k\}_{k=1}^N$ is a set of $N$ samples such that it consists of a sample $\bm{x}_i$ which is semantically similar to $\bm{x}_j$ and dissimilar to all the other samples in the set. Then the InfoNCE tries to maximize the similarity between the positive pair and minimize the similarity between the negative pairs, and is defined as:
\begin{equation}
\label{eq:loss}
    \ell_{i,j} = -\log \frac{\exp(\mathrm{sim}(\bm{h}_i,  \bm{h}_j))}{\sum_{k=1}^{N} \one{k \neq i}\exp(\mathrm{sim}(\bm{h}_i, \bm{h}_k))}
\end{equation}


\section{Domain-Agnostic Contrastive Learning with Mixup}
\label{mcl}

For domains where natural data augmentation methods are not available, we propose to apply Mixup \citep{mixup} based data interpolation for creating positive and negative samples. Given a data distribution $\mathcal{D}=\{\bm{x}_k\}_{k=1}^K$, a positive sample for an anchor $\bm{x}$ is created by taking its random interpolation with another randomly chosen sample $\tilde{ \bm{x}}$ from $\mathcal{D}$:
\begin{equation}
\label{eq:mcl}
    \bm{x}^{+} = \lambda \bm{x}+ (1-\lambda) \tilde{\bm{x}} 
\end{equation}
where $\lambda$ is a coefficient sampled from a random distribution such that $\bm{x}^{+}$ is closer to $\bm{x}$ than $\tilde{\bm{x}}$. For instance, we can sample $\lambda$ from a uniform distribution $\lambda \sim U(\alpha, 1.0)$ with high values of $\alpha$ such as 0.9. Similar to SimCLR \citep{Chen2020ASF}, positive samples corresponding to other anchor samples in the training batch are used as the negative samples for $\bm{x}$. 

Creating positive samples using Mixup in the input space (Eq. \ref{eq:mcl}) is not feasible in domains where data has a non-fixed topology, such as sequences, trees, and graphs. For such domains, we create positive samples by  mixing fixed-length hidden representations of samples \citep{manifold_mixup}. Formally, let us assume that there exists an encoder function $h: \mathcal{I} \mapsto \bm{h}$ that maps a sample  $\mathcal{I}$ from such domains to a representation $\bm{h}$ via an intermediate layer that has a fixed-length hidden representation $\bm{v}$, then  we create positive sample in the intermediate layer as:
\begin{equation}
\label{eq:mcl_hidden}
    \bm{v}^{+} = \lambda \bm{v}+ (1-\lambda) \tilde{\bm{v}} 
\end{equation}
The above Mixup based method for constructing positive samples can be interpreted as adding  noise to a given sample in the direction of another sample in the data distribution. We term this as Mixup-noise. One might ask how  Mixup-noise is a better choice for contrastive learning than other forms of noise?
The central hypothesis of our method is that a network is forced to learn better features if the noise captures the structure of the data manifold rather than being independent of it. Consider an image $\bm{x}$ and adding  Gaussian-noise to it for constructing the positive sample: $\bm{x}^+ = \bm{x}+ \bm{\delta}$, where 
$\bm{\delta} \sim \mathcal{N}(\bm{0},\bm{\sigma^2\bm{I}})$. In this case, to maximize the similarity between $\bm{x}$ and $\bm{x}^+$, the network can  learn just to take an average over the neighboring pixels to remove the noise, thus bypassing learning the semantic concepts in the image. Such kind of trivial feature transformation is not possible with Mixup-noise, and hence it enforces the network to learn better features. In addition to the aforementioned hypothesis, in Section \ref{sec::theory}, we formally conduct a theoretical analysis to understand the effect of using Gaussian-noise vs Mixup-noise in the contrastive learning framework.


For experiments, we closely follow the encoder and projection-head architecture, and the process for computing the "normalized and temperature-scaled InfoNCE loss" from SimCLR \citep{Chen2020ASF}. Our approach for Mixup-noise based \underline{D}omain-\underline{A}gnostic \underline{C}ontrastive \underline{L}earning (DACL) in the input space is  summarized in Algorithm \ref{alg:main}. Algorithm for \dacl in hidden representations can be easily derived from Algorithm \ref{alg:main} by applying mixing in Line 8 and 14 instead of line 7 and 13.

\begin{algorithm}[!t]
\caption{\label{alg:main} Mixup-noise Domain-Agnostic Contrastive Learning.}
\begin{algorithmic}[1]
    \STATE \textbf{input:} batch size $N$, temperature $\tau$, encoder function $h$, projection-head $g$, hyperparameter $\alpha$.
    \FOR{sampled minibatch $\{\bm x_k\}_{k=1}^N$}
    \STATE \textbf{for all} $k\in \{1, \ldots, N\}$ \textbf{do}
        \STATE $~~~~$\textcolor{gray}{\# Create first positive sample using Mixup Noise}
        \STATE $~~~~$$ \lambda_1 \sim U(\alpha,1.0)$
        \textcolor{gray}{~~~~~~~\# sample mixing coefficient}
        \STATE $~~~~$$ \bm{x} \sim \{\bm x_k\}_{k=1}^N - \{\bm{x_k}\}$
        \STATE $~~~~$$\tilde{\bm x}_{2k-1} = \lambda_1\bm{x_k} + (1-\lambda_1)\bm{x}$
        \STATE $~~~~$$\bm h_{2k-1} = h(\tilde{\bm x}_{2k-1})$  \textcolor{gray}{~~~~~~~~~~~~~~~~~~~~~\# apply encoder}
        \STATE $~~~~$$\bm z_{2k-1} = g({\bm h}_{2k-1})$  \textcolor{gray}{~~~~~~~~~~\# apply projection-head}
        \STATE $~~~~$\textcolor{gray}{\# Create second positive sample using Mixup Noise}
        \STATE $~~~~$$ \lambda_2 \sim U(\alpha,1.0)$
        \textcolor{gray}{~~~~~~~~~~~\# sample mixing coefficient}
        \STATE $~~~~$$ \bm{x} \sim \{\bm x_k\}_{k=1}^N - \{\bm{x_k}\}$
        \STATE $~~~~$$\tilde{\bm x}_{2k-1} = \lambda_2\bm{x_k} + (1-\lambda_2)\bm{x}$
        \STATE $~~~~$$\bm h_{2k} = h(\tilde{\bm x}_{2k})$      \textcolor{gray}{~~~~~~~~~~~~~~~~~~~~~~~~~~~~~~~~\# apply encoder}
        \STATE $~~~~$$\bm z_{2k} = g({\bm h}_{2k})$      \textcolor{gray}{~~~~~~~~~~~~~~~~~~~\# apply projection-head}
    \STATE \textbf{end for}
    \STATE \textbf{for all} $i\in\{1, \ldots, 2N\}$ and $j\in\{1, \dots, 2N\}$ \textbf{do}
    \STATE $~~~~$ $s_{i,j} = \bm z_i^\top \bm z_j / (\lVert\bm z_i\rVert \lVert\bm z_j\rVert)$ \textcolor{gray}{~~~~~~~~\# pairwise similarity}\\
    \STATE \textbf{end for}
    \STATE \textbf{define}  $\ell(i, j) \!=\! -\log \frac{\exp(s_{i,j}/\tau)}{\sum_{k=1}^{2N} \one{k \neq i}\exp(s_{i, k}/\tau)}$ \\ \STATE $\mathcal{L} = \frac{1}{2N} \sum_{k=1}^N \left[ \ell(2k\!-\!1, 2k) + \ell(2k, 2k\!-\!1)\right]$
    \STATE update networks $h$ and $g$ to minimize $\mathcal{L}$
    \ENDFOR
    \STATE \textbf{return} encoder function $h(\cdot)$, and projection-head $g(\cdot)$
\end{algorithmic}
\end{algorithm}

\subsection{Additional Forms of Mixup-Based Noise}
\label{mcl+}
We have thus far proposed the contrastive learning method using the linear-interpolation Mixup. Other forms of Mixup-noise can also be used to obtain more diverse samples for contrastive learning. In particular, we explore ``Geometric-Mixup'' and ``Binary-Mixup'' based noise. In Geometric-Mixup, we create a positive sample corresponding to a sample $\bm{x}$ by taking its weighted-geometric mean with another randomly chosen sample $\tilde{ \bm{x}}$:
\begin{equation}
\label{eq:mix_geo}
    \bm{x}^{+} =  \bm{x}^{\lambda}\odot \tilde{\bm{x}}^{(1-\lambda)} 
\end{equation}
Similar to Linear-Mixup in Eq.\ref{eq:mcl} , $\lambda$ is sampled from a uniform distribution $\lambda \sim U(\beta, 1.0)$ with high values of $\beta$.

In Binary-Mixup \citep{amr}, the elements of $\bm{x}$ are swapped with the elements of another randomly chosen sample $\tilde{\bm{x}}$. This is implemented by sampling a binary mask $\mathbf{m} \in \{0,1\}^{k}$ (where $k$ denotes the number of input features) and performing the following operation:
\begin{equation}
\label{eq:mix_fm}
    \bm{x}^{+} =  \bm{x}\odot\mathbf{m} +  \tilde{\bm{x}}\odot(1-\mathbf{m})
\end{equation}
where elements of  $\mathbf{m}$ are sampled from a $\text{Bernoulli}(\rho)$ distribution with high $\rho$ parameter.

We extend the \dacl procedure with the aforementioned additional Mixup-noise functions as follows. For a given sample $\bm{x}$, we randomly select a noise function from Linear-Mixup, Geometric-Mixup, and Binary-Mixup, and apply this function to create both of the positive samples corresponding to $\bm{x}$ (line 7 and 13 in Algorithm \ref{alg:main}). The rest of the details are the same as Algorithm \ref{alg:main}. We refer to this procedure as \daclp in the following experiments.

\section{Theoretical Analysis}
\label{sec::theory}
In this section, we mathematically analyze and compare the  properties of  Mixup-noise and Gaussian-noise based contrastive learning for a binary classification task. We first prove that for both Mixup-noise and Gaussian-noise, optimizing hidden layers with  a contrastive loss is related to minimizing classification loss with the last layer being optimized using labeled data. We then prove that the proposed method with Mixup-noise induces a different regularization effect on the classification loss when compared with that of Gaussian-noise. The difference in  regularization effects shows the advantage of Mixup-noise over Gaussian-noise when the data manifold lies in a low dimensional subspace. Intuitively, our theoretical results show that contrastive learning with Mixup-noise has implicit data-adaptive regularization effects that promote generalization.

To compare the cases of Mixup-noise and Gaussian-noise, we focus on  linear-interpolation based Mixup-noise and unify the two cases using the following observation. For Mixup-noise,
we can write
$
\xb^{+}_\mix = \lambda \xb+ (1-\lambda) \tilde \xb =\xb+\alpha\delta(\xb,\tilde \xb)  
$
with $\alpha= 1-\lambda>0$ and  $\delta(\xb,\tilde \xb)= (\tilde \xb -\xb)$ where $\tilde \xb$ is drawn from some (empirical) input data distribution. For Gaussian-noise,
we can write $
\xb^{+}_\gauss = \xb + \alpha\delta(\xb,\tilde \xb)
$ 
with $\alpha>0$  and $\delta(\xb, \tilde \xb)=  \tilde \xb$ where $  \tilde \xb$ is drawn from some Gaussian distribution. Accordingly, for each  input $\xb$, we can write  the positive example pair $(\xb^{+}_{}, \xb^{++})$ and the negative example $\xb^{-}$ for both cases as:   $\xb^{+} = \xb + \alpha\delta(\xb,\tilde \xb)$, $\xb^{++} = \xb + \alpha'\delta(\xb,\tilde \xb')$, and $\xb^{-} =\bar \xb + \alpha''\delta(\bar \xb,\tilde \xb'')$, where $\bar \xb $ is another input sample. Using this unified notation, we theoretically analyze our method with the standard contrastive loss $\ell_\cont$ defined by
$
\ell_\cont(\xb^{+}_{}, \xb^{++}, \xb^{-}) =- \log \frac{\exp( \simi[h(\xb^{+}_{}), h( \xb^{++}_{})])}{\exp(\simi[ h(\xb^{+}), h( \xb^{++}_{})])+\exp(\simi[h(\xb^{+}_{}), h(\xb^{-})] )},
$
where $h(\xb)\in \RR^d$ is the output of the last hidden layer and $\simi[q, q']=\frac{q\T q'}{\|q\| \|q'\|}$ for any given vectors $q$ and $q'$. This contrastive loss $\ell_\cont$ without the projection-head $g$ is commonly used in practice and captures the essence of   contrastive learning. Theoretical analyses of the benefit of the projection-head $g$ and other forms of Mixup-noise are left to future work. 

This section focuses on binary classification with $y \in \{0,1\}$  using the standard binary cross-entropy loss:
$
\ell_{\class}(q, y)= - y \log(\hat p_{q}(y=1)) - (1-y) \log( \hat p_{q}(y=0))
$
with  $\hat p_{q}(y=0)= 1- \hat p_{q}(y=1)$ where   $
\hat p_{q}(y=1)= \frac{1}{1+\exp(-q )}. 
$ We use $f(\xb)=h(\xb)\T w$ to represent the output of the classifier for some $w$; i.e., $\ell_{\class}(f(\xb), y)$ is the cross-entropy loss of the classifier $f$ on the sample $(\xb,y)$. Let  $\phi : \RR \rightarrow [0,1]$ be any Lipschitz function with constant $L_{\phi}$ such that  $\phi(q) \ge \one{q \le 0}$ for all $q\in \RR$; i.e., $\phi$ is an smoothed version of 0-1 loss. For example, we can  set  $\phi$ to be the hinge loss. Let $\Xcal \subseteq \RR^d$ and $\Ycal$ be the input and output spaces as $\xb \in \Xcal$ and $y \in \Ycal$.
Let  $c_{\xb}$ be a real number such that $c_{\xb}\ge (\xb_{k})^2$ for all $\xb\in \Xcal$ and $k \in\{1,\dots, d\}$.

As we aim to compare the cases of Mixup-noise and Gaussian-noise accurately (without taking loose bounds), we first prove an exact relationship between the contrastive loss and classification loss. That is, the following theorem shows that optimizing hidden layers with contrastive loss $\ell_\cont(\xb^{+},\xb^{++},\xb^{-})$ is related to minimising classification loss $\ell_{\class}\left(f(\xb^+), y \right)$ with the error term $\EE_y[\left(1-\bar \rho({y}) )\Ecal_{y}\right]$, where the error term increases as the probability of the negative example $x^{-}$ having the same label as that of the positive example $x^{+}$ increases:    
\begin{theorem} \label{thm:1}
Let $\Dcal$ be a  probability distribution over $(\xb,y)$ as $(\xb, y) \sim \Dcal$, with the corresponding marginal  distribution $\Dcal_x$ of $\xb$ and  conditional distribution $\Dcal_y$ of $\xb$ given a $y$. Let $\bar \rho(y)=\EE_{{(\xb',  y') \sim \Dcal}  }[\one{y' \neq y}]$ $(=\Pr(y' \neq y \mid y)>0)$. Then, for  any distribution pair $(\Dcal_{\tilde \xb}, \Dcal_{\alpha})$ and function $\delta$, the following holds: 
\begin{align*}
& \EE_{\hspace{-2pt}\substack{\xb, \bar \xb\sim \Dcal_{x}, \\ \tilde \xb, \tilde \xb',\tilde \xb'' \sim \Dcal_{\tilde x}, \\ \alpha, \alpha', \alpha''\sim \Dcal_{\alpha}}} [\ell_\cont(\xb^{+},\xb^{++},\xb^{-}) ]
\\ & =  \scalebox{0.96}{$\displaystyle   \EE_{\substack{(\xb, y) \sim \Dcal, \bar \xb\sim D_{\bar y}, \\ \tilde \xb, \tilde \xb',\tilde \xb'' \sim \Dcal_{\tilde x}, \\ \alpha,\alpha', \alpha''\sim \Dcal_{\alpha}}} \left[ \rho({y})\ell_{\class}\left(f(\xb^+)   , y \right) \right]+\EE_y[\left(1-\bar \rho({y}) )\Ecal_{y}\right]$}
\end{align*}
where 
$$
\scalebox{0.85}{$\displaystyle \Ecal_{y}=\EE_{\hspace{-5pt}\substack{\xb, \bar \xb\sim \Dcal_{y}, \\ \tilde \xb, \tilde \xb',\tilde \xb'' \sim \Dcal_{\tilde x}, \\ \alpha, \alpha', \alpha''\sim \Dcal_{\alpha}}}  \left[ \log\left(1+  e^{-\tfrac{h(\xb^+)\T }{\|h(\xb^+)\|} \left(\tfrac{h(\xb^{++})}{\|h(\xb^{++})\|}- \tfrac{h(\xb^-)}{\|h(\xb^-)\|} \right)} \right) \right] $},
$$ 
$f(\xb^+)= h(\xb^{+} )\T  \tilde w$, $\bar y =1-y$,
$
\tilde w = \|h(\xb^{+} )\|^{-1} (\| \allowbreak h(  \pi_{y,1}(\allowbreak \xb^{++}, \xb^- ))\|^{-1} h(\pi_{y,1}(\xb^{++}, \xb^- )) -  \|h(\pi_{y,0}(\xb^{++},  \allowbreak \xb^- ))\|^{-1} h(\pi_{y,0}(\xb^{++}, \xb^- )))
$, and
$\pi_{y,y'}(\xb^{++}, \xb^- )=\one{y = y'} \xb^{++} +(1-\one{y = y'})\xb^- $.

\end{theorem}

All  the proofs are presented in Appendix \ref{app:proof}. Theorem \ref{thm:1} proves the exact relationship for \textit{training} loss when we set the distribution  $\Dcal$ to be an empirical distribution with Dirac measures on training data points: see Appendix \ref{app:discuss_theory} for more details. In general, Theorem \ref{thm:1} relates optimizing the  contrastive loss  $\ell_\cont(\xb^{+}, \xb^{++}, \xb^{-})$ to   minimizing the   classification loss $\ell_{\class}\left(f(\xb^+), y_{i} \right)$ at the perturbed sample $\xb^{+}$. The following theorem then shows that it is
approximately minimizing  the   classification loss $\ell_{\class}\left(f(\xb), y_{i} \right)$ at the original sample $\xb$ with additional regularization terms on $\nabla f(\xb)$:

\begin{theorem} \label{thm:2}
Let $\xb$ and  $w$ be vectors such that $\nabla f(\xb) $ and $\nabla^{2} f(\xb)$ exist. Assume that   $f(\xb)=\nabla f(\xb)\T \xb$, $\nabla^2 f(\xb)=0$, and   $\EE_{\tilde \xb\sim \Dcal_{\tilde x}} [\tilde \xb] =0$. Then, if $y f(\xb) +(y-1)f(\xb)\ge 0$, the following two statements hold for any $\Dcal_{\tilde x}$ and $\alpha>0$:
\vspace{-5pt}
\begin{description}[leftmargin=\parindent]
\item[(i)] 
\emph{(Mixup)} if $\delta(\xb,\tilde \xb)= \tilde \xb - \xb$, \vspace{-0pt}
\begin{align} \label{eq:thm2:1} 
&\EE_{\tilde \xb\sim \Dcal_{\tilde x}}[\ell_{\class}(f(\xb^+), y )] 
\\ \nonumber &= \ell_{\class}(f(\xb), y ) + c_{1}(\xb)| \| \nabla f(\xb)\|_{} +c_2(\xb)\| \nabla f(\xb)\|^{2} _{}
\\ \nonumber & \hspace{60pt} + c_3(\xb)\| \nabla f(\xb) \|_{\EE_{\tilde \xb\sim \Dcal_{\tilde x}}[ \tilde \xb \tilde \xb\T]}^2+ O(\alpha^3),  
\end{align}
\item[(ii)] \vspace{-5pt}
\emph{(Gaussian-noise)} if $\delta(\xb,\tilde \xb)=  \tilde \xb \sim \Ncal(0,\sigma^2 I)$, \vspace{-0pt}
\begin{align} \label{eq:thm2:2} 
& \EE_{\tilde \xb \sim \Ncal(0,\sigma^2 I)}[\ell_{\class}\left(f(\xb^+), y \right) ]
\\ \nonumber & =\ell_{\class}(f(\xb), y ) +\sigma^2 c_{3}(\xb) \| \nabla f(\xb)\|^{2}+ O(\alpha^3),
\end{align}
\end{description}
\vspace{-5pt}
where
$
c_1(\xb) = \alpha |\cos( \nabla f(\xb), \xb )|| y-\psi(f(\xb))|  \|\xb\|\ge 0
$, 
$
c_2(\xb)=\frac{\alpha^2 |\cos( \nabla f(\xb), \xb )|^{2}  \|\xb\|   }{2}     |\psi'(f(\xb))|\ge 0 
$, and 
$
c_3(\xb) =\frac{\alpha^2}{2}     |\psi'(f(\xb))|> 0$. Here, $\psi$ is the logic function as $
\psi(q)= \frac{\exp(q)}{1+\exp(q)}$ ($\psi'$ is its derivative),    $\cos(a, b)$ is the cosine similarity of two vectors $a$ and $b$, and 
  $\|v \|_{M}^2=v\T M v$ for any   positive semidefinite matrix $M$.\footnote{We use this notation for conciseness without assuming that it is a norm. If $M$ is only positive semidefinite instead of positive definite, $\|\cdot \|_{M}$ is not a norm since this does not satisfy the definition of the norm for positive definiteness; i.e., $\|v\|=0$ does not imply $v=0$. }
\end{theorem}  

The assumptions of $f(\xb)=\nabla f(\xb)\T \xb$ and $\nabla^2 f(\xb)=0$ in Theorem \ref{thm:2} are satisfied by  feedforward deep neural networks with ReLU and max pooling (without skip connections) as well as by linear models.
The condition of $y f(\xb) +(y-1)f(\xb)\ge 0$ is satisfied whenever the  training sample $(\xb, y)$ is classified correctly. In other words, Theorem \ref{thm:2} states that when the model  classifies a training sample  $(\xb, y)$ correctly, a training algorithm  implicitly minimizes the additional regularization terms for the sample $(\xb, y)$, which partially explains the benefit of training after correct classification of training samples.

In Eq. \eqref{eq:thm2:1}--\eqref{eq:thm2:2}, we can see that both  the Mixup-noise  and Gaussian-noise versions have different regularization effects on $\|\nabla f(\xb)\|$ --- the Euclidean norm of the gradient of the model $f$ with respect to input $\xb$. 
In the case of  the linear model, we  know from previous work that the regularization on $\|\nabla f(\xb)\|= \|w\|$ indeed promotes generalization: 

\begin{remark} \label{thm:3} 
\emph{(\citealp{bartlett2002rademacher})} Let $\Fcal_b = \{\xb\mapsto w\T \xb : \|w\|^2\le b\}$. Then,  for any $\delta>0$, with probability at least $1-\delta$ over an i.i.d. draw of $n$  examples $((\xb_i,y_i))_{i=1}^n$, the following holds for all $f \in \Fcal_b$: \vspace{-8pt}
\begin{align} \label{eq:thm:3}
\nonumber  & \EE_{(\xb,y)}[\one{(2y-1) \neq \sgn(f(\xb)) }]
 -\frac{1}{n} \sum_{i=1}^n \phi((2y_{i}-1)f(\xb_{i}))
 \\  & \le  4L_{\phi}\sqrt{\frac{ bc_{\xb} d}{n}}+  \sqrt{\frac{\ln(2/\delta)}{2n}}.
\end{align}
\end{remark}
By comparing Eq. \eqref{eq:thm2:1}--\eqref{eq:thm2:2} and by setting $\Dcal_{\tilde x}$ to be the  input data distribution, we can see that the Mixup-noise version has additional regularization effect on $\|\nabla f(\xb) \|_{\Sigma_X}^2=\|w\|_{\Sigma_X}^2$, while the Gaussian-noise version does not, where $\Sigma_X=\EE_x[\xb \xb\T]$ is the input covariance matrix. The following theorem shows that this implicit   regularization with the Mixup-noise version can further reduce the generalization error:

\begin{theorem} \label{thm:4}
Let $
\Fcal_b^{(\mix)} = \{\xb\mapsto w\T \xb : \|w\|_{\Sigma_X}^2 \le b\}
$. Then,  for any $\delta>0$, with probability at least $1-\delta$ over an iid draw of $n$  examples $((\xb_i,y_i))_{i=1}^n$, the following holds for all $f \in \Fcal_b^{(\mix)}$:
\begin{align} \label{eq:thm:4}
\nonumber  & \EE_{(\xb,y)}[\one{(2y-1) \neq \sgn(f(\xb)) }]
 -\frac{1}{n} \sum_{i=1}^n \phi((2y_{i}-1)f(\xb_{i})) 
 \\ & \le  4L_{\phi}\sqrt{\frac{ b \rank(\Sigma_X)}{n}}+  \sqrt{\frac{\ln(2/\delta)}{2n}}.
\end{align}
\end{theorem}

Comparing  Eq. \eqref{eq:thm:3}--\eqref{eq:thm:4}, we can see that the proposed method with Mixup-noise has the  advantage over the Gaussian-noise when the  input data distribution lies in low dimensional manifold  as $
\rank(\Sigma_X)<d$. In general, our theoretical results  show that the proposed method with  Mixup-noise induces the implicit regularization on $\| \nabla f(\xb) \|_{\Sigma_X}^2$, which can reduce the complexity of the model class of $f$ along the data manifold captured by the covariance $\Sigma_X$. See Appendix \ref{app:discuss_theory} for additional discussions on the interpretation of Theorems  \ref{thm:1} and \ref{thm:2} for  neural networks. 

The proofs of Theorems \ref{thm:1} and \ref{thm:2} hold true also when we set $\xb$ to be the output of a hidden layer and by redefining the domains of $h$ and $f$ to be  the output of the hidden layer. Therefore, by treating $\xb$ to be the output of a hidden layer, our theory also applies to the contrastive learning with positive samples created by mixing the hidden representations of samples. In this case, Theorems \ref{thm:1} and \ref{thm:2} show that the contrastive learning method implicitly regularizes $\| \nabla f(\xb^{(l)}) \|_{\EE[ \tilde \xb^{(l)}(\tilde \xb^{(l)})\T]}$  --- the norm of the gradient of the model $f$ with respect to the output   $\xb^{(l)}$  of the $l$-th  hidden layer in the direction of data manifold. Therefore, contrastive learning with Mixup-noise at the input space or a hidden space can promote generalization in the data manifold in the input space or the hidden space.

\section{Experiments}
\label{sec:experiments}
We present results on three different application domains: tabular data, images, and graphs. 
For all datasets, to evaluate the learned representations under different contrastive learning methods, we use the
linear evaluation protocol \citep{bachman2019learning,cpc_v2,He2019MomentumCF,Chen2020ASF}, where a linear classifier is trained on top of a frozen encoder network, and the test accuracy is used as a proxy for representation quality. Similar to SimCLR, we discard the projection-head during linear evaluation.

 For each of the experiments, we give details about the architecture and the experimental setup in the corresponding section. In the following, we describe common hyperparameter search settings. For experiments on tabular and image datasets (Section \ref{exp:tabular} and \ref{exp:image}), we search the hyperparameter $\alpha$ for linear mixing (Section \ref{mcl} or line 5 in Algorithm \ref{alg:main}) from the set $\{0.5, 0.6, 0.7, 0.8, 0.9\}$. To avoid the search over hyperparameter $\beta$ (of Section \ref{mcl+}), we set it to same value as $\alpha$. For the hyperparameter $\rho$ of Binary-Mixup (Section \ref{mcl+}), we search the value from the set $[0.1, 0.3, 0.5]$. For Gaussian-noise based contrastive learning, we chose the mean of Gaussian-noise from the set $\{0.05, 0.1, 0.3, 0.5\}$ and the standard deviation is set to $1.0$. For all experiments, the hyperparameter temperature $\tau$ (line 20 in Algorithm \ref{alg:main}) is searched from the set $\{0.1, 0.5, 1.0\}$. For each of the experiments, we report the best values of aforementioned hyperparameters in the Appendix \ref{app:hyper} .

For experiments on graph datasets (Section \ref{exp:graph}), we fix the value of $\alpha$ to $0.9$ and value of temperature $\tau$ to $1.0$.

\subsection{Tabular Data}
\label{exp:tabular}

For tabular data experiments, we use Fashion-MNIST and CIFAR-10 datasets as a proxy by permuting the pixels and flattening them into a vector format. We use No-pretraining and Gaussian-noise based contrastive leaning as baselines. Additionally, we report supervised learning results (training the full network in a supervised manner).

We use a 12-layer fully-connected network as the base encoder and a 3-layer projection head, with ReLU non-linearity and batch-normalization for all layers.  All pre-training methods are trained for 1000 epochs with a batch size of 4096. The linear classifier is trained for 200 epochs with a batch size of 256. We use LARS optimizer~\cite{you2017large} with cosine decay schedule without restarts \citep{loshchilov}, for both pre-training and linear evaluation. The initial learning rate for both pre-training and linear classifier is set to 0.1.

\textbf{Results:} As shown in Table \ref{tab:fcn}, \dacl performs significantly better than the Gaussian-noise based contrastive learning. \daclp, which uses additional Mixup-noises (Section \ref{mcl+}), further improves the performance of \daclns. More interestingly, our results show that the linear classifier applied to the representations learned by \dacl gives better performance than training the full network in a supervised manner. 

\begin{table}[h]
\centering
\begin{tabular}{l l l }
\toprule
{\bf Method} & {\bf Fashion-MNIST} & {\bf CIFAR10}  \\ \midrule
No-Pretraining & 66.6 & 26.8\\ 
Gaussian-noise & 75.8   & 27.4  \\ 
\dacl & 81.4   & 37.6 \\ 
\daclp & \textbf{82.4} & \textbf{39.7} \\ 
\hline
Full network \\ supervised training & 79.1 & 35.2 \\ 
\bottomrule
\end{tabular}
\caption{Results on tabular data with a 12-layer fully-connected network.}
\label{tab:fcn}
\end{table}

\subsection{Image Data}
\label{exp:image}

 We use three benchmark image datasets: CIFAR-10, CIFAR-100, and ImageNet. For CIFAR-10 and CIFAR-100, we use No-Pretraining, Gaussian-noise based contrastive learning and SimCLR \citep{Chen2020ASF} as baselines. For ImageNet, we use recent contrastive learning methods e.g. \cite{gidaris2018unsupervised,donahue2019large, bachman2019learning,tian2019contrastive, He2019MomentumCF,cpc_v2} as additional baselines. SimCLR+\dacl refers to the combination of the SimCLR and \dacl methods, which is implemented using the following steps: (1) for each training batch, compute the SimCLR loss and \dacl loss separately and (2) pretrain the network using the sum of SimCLR and \dacl losses.

 For all experiments, we closely follow the details in SimCLR \citep{Chen2020ASF}, both for pre-training and  linear evaluation. We use ResNet-50(x4) \citep{He2016DeepRL} as the base encoder network, and a 3-layer MLP projection-head to project the representation to a 128-dimensional latent space.

\textbf{Pre-training:} For SimCLR and SimCLR+\dacl pretraining, we use the following augmentation operations:  random crop and resize (with random flip), color distortions, and Gaussian blur. We train all models with a batch size of 4096 for 1000 epochs for CIFAR10/100 and 100 epochs for ImageNet.\footnote{Our reproduction of the results of SimCLR for ImageNet in Table \ref{tab::imagenet} differs from \cite{Chen2020ASF} because our experiments are run for 100 epochs vs  their 1000 epochs.} We use LARS optimizer with learning rate 16.0 $(=1.0\times \text{Batch-size}/256)$ for CIFAR10/100 and 4.8 $(=0.3\times \text{Batch-size}/256)$ for ImageNet. Furthermore, we use linear warmup for the first 10 epochs
and decay the learning rate with the cosine decay schedule without restarts \citep{loshchilov}.  The weight decay is set to  $10^{-6}$.

\textbf{Linear evaluation:} 
To stay domain-agnostic, we do not use any data augmentation during the linear evaluation of No-Pretraining, Gaussian-noise, \dacl and \daclp methods in Table \ref{tab::cnn_cifar10_100} and \ref{tab::imagenet}. For linear evaluation of SimCLR and SimCLR+\daclns, we use random cropping with random left-to-right flipping, similar to \cite{Chen2020ASF}. For CIFAR10/100, we use a batch-size of 256 and train the model for 200 epochs, using LARS optimizer with learning rate 1.0 $(=1.0\times \text{Batch-size}/256)$ and cosine decay schedule without restarts. For ImageNet, we use a batch size of 4096 and train the model for 90 epochs, using LARS optimizer with learning rate 1.6 $(=0.1\times \text{Batch-size}/256)$ and cosine decay schedule without restarts. For both the CIFAR10/100 and ImageNet, we do not use weight-decay and learning rate warm-up.

\textbf{Results:} 
We present the results for CIFAR10/CIFAR100 and ImageNet in Table \ref{tab::cnn_cifar10_100}  and Table \ref{tab::imagenet} respectively. We observe that \dacl is better than  Gaussian-noise based contrastive learning by a wide margin and \daclp can improve the test accuracy even further. However, \dacl falls short of methods that use image augmentations such as SimCLR \citep{Chen2020ASF}. This shows that the invariances learned using the image-specific augmentation methods (such as cropping, rotation, horizontal flipping) facilitate learning better representations than making the representations invariant to Mixup-noise. This opens up a further question: are the invariances learned from image-specific augmentations complementary to the Mixup-noise based invariances? To answer this, we combine \dacl with SimCLR (SimCLR+\dacl in Table \ref{tab::cnn_cifar10_100} and Table \ref{tab::imagenet}) and show that it can improve the performance of SimCLR across all the datasets. This suggests that Mixup-noise is complementary to other image data augmentations for contrastive learning. 

\begin{table}[h]
\centering
\begin{tabular}{l l l }
\toprule
\textbf{Method}	& \textbf{CIFAR-10} & \textbf{CIFAR-100}  \\ 
\midrule
No-Pretraining & 43.1 & 18.1 \\ 
Gaussian-noise & 56.1 & 29.8  \\ 
\dacl  & 81.3 &  46.5 \\ 
\daclp & 83.8 &  52.7  \\ 
SimCLR & 93.4  & 73.8  \\ 
SimCLR+\dacl & \textbf{94.3}   & \textbf{75.5} \\
\bottomrule
\end{tabular}
\caption{Results on CIFAR10/100 with ResNet50($4\times$)}
\label{tab::cnn_cifar10_100}
\end{table}

\begin{table}[h]
\centering
\resizebox{\linewidth}{!}{
\begin{tabular}{lllll}
\toprule
\textbf{Method} & \textbf{Architecture}         & \textbf{Param(M)} & \textbf{Top 1} & \textbf{Top 5} \\ \midrule
Rotation \citep{gidaris2018unsupervised} & ResNet50 ($4\times$)  & 86            & 55.4  &    -   \\ 
BigBiGAN \\ \citep{donahue2019large}  & ResNet50 ($4\times$)  & 86            & 61.3  &    81.9   \\ 
AMDIM \\ \citep{bachman2019learning}  & Custom-ResNet           & 626            & 68.1  &    -   \\
CMC \citep{tian2019contrastive}  & ResNet50 ($2\times$) & 188            & 68.4  &    88.2   \\
MoCo \citep{He2019MomentumCF}  & ResNet50 ($4\times$) & 375             & 68.6  &    -   \\
CPC v2 \citep{cpc_v2}& ResNet161          & 305            & 71.5  & 90.1  \\
BYOL (300 epochs) \\\citep{byol}  & ResNet50 ($4\times$) & 375 & 72.5 & 90.8 \\ 
\midrule
No-Pretraining  & ResNet50 ($4\times$)& 375  & 4.1  & 11.5  \\ 
Gaussian-noise    & ResNet50 ($4\times$) & 375    & 10.2       &   23.6     \\ 
\dacl    & ResNet50 ($4\times$) & 375            &  24.6   & 44.4   \\ 
SimCLR \citep{Chen2020ASF}   & ResNet50 ($4\times$) & 375            & 73.4     & 91.6    \\ 
SimCLR+\dacl   & ResNet50 ($4\times$) & 375            &  \textbf{74.4}   & \textbf{92.2}    \\ 
\bottomrule
\end{tabular}
}
\caption{Accuracy of linear classifiers trained on representations learned with different self-supervised methods on the ImageNet dataset.}
\label{tab::imagenet}
\end{table}

\subsection{Graph-Structured Data} 
\label{exp:graph}

\begin{table*}[t]
\small
\centering
\setlength{\tabcolsep}{5pt}
\resizebox{\textwidth}{!}{\begin{tabular}{lllllll}
\toprule

\textbf{Dataset} & MUTAG & PTC-MR  &  REDDIT-BINARY &  REDDIT-M5K &	 IMDB-BINARY &	 IMDB-MULTI  \\
\midrule
No. Graphs &     {$188$} &	 {$344$}  &  {$2000$} & 	 {$4999$} &	 {$1000$} &	 {$1500$} \\  
		
No. classes &    {$2$} &	 {$2$}  &  {$2$} & 	 {$5$} &	 {$2$} &	 {$3$} \\  
Avg. Graph Size    &  {$17.93$} &
		{$14.29$}  &  {$429.63$} & 	 {$508.52$} &	 {$19.77$} &	 {$13.00$} \\
 \midrule
& & &  \textbf{Method} & & &  \\
 \midrule
		No-Pretraining &$81.70 \pm 2.58$ & $53.07 \pm 1.27$ & $55.13 \pm 1.86$ &$24.27\pm0.93$ &$52.67\pm 2.08$ & $33.72 \pm 0.80$ 
\\		
		InfoGraph \citep{infoGraph}     &$\mathbf{86.74 \pm 1.28}$ & $57.09 \pm 1.52$ & $63.52\pm 1.66$ & $\mathbf{42.89\pm0.62}$ & $63.97 \pm 2.05$ & $39.28 \pm 1.43$
 \\  
       	\dacl     &$85.31 \pm 1.34$ & $\mathbf{59.24\pm 2.57}$ & $\mathbf{66.92\pm 3.38}$ &$42.86\pm 1.11$ & $\mathbf{64.71\pm2.13}$ & $\mathbf{40.16\pm 1.50}$ 
 \\  \hline

\end{tabular}}
\caption{Classification accuracy  using a linear classifier trained on representations obtained using different self-supervised methods on 6 benchmark graph classification datasets.} 
    \label{table:graph_classifiction}
\end{table*}

We present the results of applying \dacl to  graph classification problems  using six well-known benchmark datasets: MUTAG, PTC-MR, REDDIT-BINARY, REDDIT-MULTI-5K, IMDB-BINARY, and IMDB-MULTI \citep{simonovsky2017dynamic,yanardag2015deep}. For baselines, we use No-Pretraining and InfoGraph \citep{infoGraph}. InfoGraph is a state-of-the-art contrastive learning method for graph classification problems, which is based on maximizing the mutual-information between the global and node-level features of a graph by formulating this as a contrastive learning problem. 

For applying \dacl to graph structured data, as discussed in Section \ref{mcl}, it is required to obtain fixed-length representations from an intermediate layer of the encoder. For graph neural networks, e.g. Graph Isomorphism Network (GIN) \cite{xu2018powerful}, such fixed-length representation can be obtained by applying global pooling over the node-level representations at any intermediate layer. Thus, the Mixup-noise can be applied to any of the intermediate layer by adding an auxiliary feed-forward network on top of such intermediate layer. However, since we follow the encoder and projection-head architecture of SimCLR, we can also apply the Mixup-noise to the output of the encoder. In this work, we present experiments with Mixup-noise applied to the output of the encoder and leave the experiments with Mixup-noise at intermediate layers for future work.

We closely follow the experimental setup of InfoGraph \citep{infoGraph} for a fair comparison, except that we report results for a linear classifier instead of the Support Vector Classifier applied to the pre-trained representations.
This choice was made to maintain the coherency of evaluation protocol throughout the paper as well as with respect to the previous state-of-the-art self-supervised learning papers. 
%
\footnote{Our reproduction of the results for InfoGraph differs from \cite{infoGraph} because we apply a linear classifier instead of Support Vector Classifier on the pre-trained features.} For all the pre-training methods in Table \ref{table:graph_classifiction}, as graph encoder network, we use GIN \citep{xu2018powerful} with 4 hidden layers and node embedding dimension of 512. The output of this encoder network is a fixed-length vector of dimension $4\times512$. Further, we use a 3-layer projection-head with its hidden state dimension being the same as the output dimension of a 4-layer GIN $(4\times512)$. Similarly for InfoGraph experiments, we use a  3-layer discriminator network with hidden state dimension $4\times512$.

 For all experiments, for pretraining, we train the model for 20 epochs with a batch size of 128, and for linear evaluation, we train the linear classifier on the learned representations for 100 updates with full-batch training. For both pre-training and linear evaluation, we  use Adam  optimizer \cite{kingma2014method} with an initial learning rate chosen from the set $\{10^{-2}, 10^{-3}, 10^{-4}\}$. We perform linear evaluation using 10-fold cross-validation. Since these datasets are small in the number of samples,
 the linear-evaluation accuracy varies significantly across the pre-training epochs. Thus,  we report the average of linear classifier accuracy over the last five pre-training epochs. All the experiments are repeated five times.

\textbf{Results:} In Table \ref{table:graph_classifiction} we see that \dacl closely matches the performance of InfoGraph, with the classification accuracy of these methods being within the standard deviation of each other. In terms of the classification accuracy mean, \dacl outperforms InfoGraph on four out of six datasets. This result is particularly appealing because we have used no domain knowledge for formulating the contrastive loss, yet achieved performance comparable to a state-of-the-art graph contrastive learning method.

\section{Related Work}
\label{sec:related_work}
    \paragraph{Self-supervised learning:} Self-supervised learning methods can be categorized based on the pretext task they seek to learn. 
    For instance, in \cite{virginia}, the pretext task is to minimize the disagreement between the outputs of neural networks processing two different modalities of a given sample.
    In the following, we briefly review various pretext tasks across different domains.  In the natural language understanding, pretext tasks include, predicting the neighbouring words (word2vec \cite{word2vec}), predicting the next word~\citep{dai2015semi,peters2018deep,radford2019language}, predicting the next sentence \citep{skipthought,bert}, predicting the masked word \citep{bert,yang2019xlnet,roberta,Lan2020ALBERT}), and predicting the replaced word in the sentence \citep{Clark2020ELECTRA}. For computer vision, examples of pretext tasks include  rotation prediction \citep{gidaris2018unsupervised}, relative position prediction of image patches \citep{10.1109/ICCV.2015.167}, image colorization \citep{Zhang2016ColorfulIC}, reconstructing the original image from the partial image \citep{Pathak2016ContextEF,DBLP:journals/corr/ZhangIE16a}, learning invariant representation under image transformation \cite{Misra2020SelfSupervisedLO}, and predicting an odd video subsequence in a video sequence \citep{FernandoCVPR2017}. For graph-structured data, the pretext task can be predicting the context (neighbourhood of a given node) or predicting the masked attributes of the node \citep{Hu*2020Strategies}. 
    Most of the above pretext tasks in these methods are domain-specific, and hence they cannot be applied to other domains. Perhaps a notable exception is the language modeling objectives, which have been shown to work for both NLP and computer vision~\citep{dai2015semi,chen2020generative}.
    
    \paragraph{Contrastive learning:} Contrastive learning is a form of self-supervised learning where the pretext task is to bring positive samples closer than the negative samples in the representation space. These methods can be categorized based on how the positive and negative samples are constructed. In the following, we will discuss these categories and the domains where these methods cannot be applied: (a) this class of methods  use domain-specific augmentations \citep{chopra, hadsell, Ye_2019_CVPR, He2019MomentumCF, Chen2020ASF,caron2020unsupervised} for creating positive and negative samples. These methods are state-of-the-art for computer vision tasks but can not be applied to domains where semantic-preserving data augmentation does not exist, such as graph-data or tabular data. (b) another class of methods constructs positive and negative samples by defining the local and global context in a sample \citep{infomax, infoGraph, velic, bachman2019learning, selfie}. These methods can not be applied to domains where such global and local context does not exist, such as tabular data. (c) yet another class of methods uses the ordering in the sequential data to construct positive and negative samples \citep{Oord2018RepresentationLW,cpc_v2}. These methods cannot be applied if the data sample cannot be expressed as an ordered sequence, such as graphs and tabular data. Thus our motivation in this work is to propose a contrastive learning method that can be applied to a wide variety of domains.
    
    
    
    \paragraph{Mixup based methods:} Mixup-based methods allow inducing inductive biases about how a model's predictions should behave in-between two or more data samples. Mixup\citep{mixup,bclearning} and its numerous variants\cite{manifold_mixup,yun2019cutmix,faramarzi2020patchup} have seen remarkable success in supervised learning problems, as well other problems such as semi-supervised learning \citep{ict, mixmatch}, unsupervised learning using autoencoders \citep{amr, acai}, adversarial learning \citep{iat,Lee2020AdversarialVM,Pang*2020Mixup}, graph-based learning \citep{verma2020graphmix,10.1145/3394486.3403063}, computer vision \citep{yun2019cutmix, jeong2020interpolationbased, panfilov2019improving}, natural language \citep{guo2019augmenting,zhang2020seqmix} and speech \citep{9054719,Tomashenko}.
    In contrastive learning setting, Mixup-based methods have been recently explored in \cite{shen2020rethinking, kalantidis2020hard,kim2020mixco}. Our work differs from aforementioned works in important aspects: unlike these methods, we theoretically demonstrate why Mixup-noise based directions are better than Gaussian-noise for constructing positive pairs, we propose other forms of Mixup-noise and show that these forms are complementary to linear Mixup-noise, and experimentally validate our method across different domains. We also note that Mixup based contrastive learning methods such as ours and \cite{shen2020rethinking, kalantidis2020hard,kim2020mixco} have advantage over recently proposed adversarial direction based contrastive learning method \cite{kim2020adversarial} because the later method requires additional gradient computation.

\section{Discussion and Future Work}
In this work, with the motivation of designing a domain-agnostic self-supervised learning method, we study Mixup-noise as a way for creating positive and negative samples for the contrastive learning formulation. Our results show that the proposed method \dacl  is a viable option for the domains where data augmentation methods are not available. Specifically, for tabular data, we show that \dacl and \daclp can achieve better test accuracy than training the neural network in a fully-supervised manner. For graph classification, \dacl is on par with the recently proposed mutual-information maximization method for contrastive learning \citep{infoGraph}. For the image datasets, \dacl falls short of those methods which use image-specific augmentations such as random cropping, horizontal flipping, color distortions, etc. However, our experiments show that the Mixup-noise in \dacl can be used as complementary to image-specific data augmentations. 
As future work, one could easily extend \dacl to other domains such as natural language and speech. From a theoretical perspective, we have analyzed \dacl in the binary classification setting, and extending this analysis to the multi-class setting might shed more light on developing a better Mixup-noise based contrastive learning method. Furthermore, since different kinds of Mixup-noise examined in this work are based only on \textit{random} interpolation between \textit{two} samples, extending the experiments by mixing between more than two samples or \textit{learning} the optimal mixing policy through an auxiliary network is another promising avenue for future research. 

\nocite{langley00}

\bibliography{egbib}
\bibliographystyle{icml2020}

\clearpage

\appendix

\onecolumn

\begin{center}
    \LARGE{Appendix}
\end{center}

\allowdisplaybreaks

\section{Additional Discussion on Theoretical Analysis} \label{app:discuss_theory}

\paragraph{On the interpretation of  Theorem \ref{thm:1}. } In Theorem \ref{thm:1}, the distribution $\Dcal$ is arbitrary. For example, if the number of  samples generated during training is finite and $n$, then the simplest way to instantiate Theorem \ref{thm:1}  is  to set $\Dcal$ to represent the empirical measure $\frac{1}{n}\sum_{i=1}^m \delta_{(x_i,y_i) }$ for training data $((x_i,y_i))_{i=1}^m$ (where the Dirac measures $\delta_{(x_i,y_i)}$), which yields  the following:
\begin{align*}
&\frac{1}{n^{2}} \sum_{i=1}^m \sum_{j=1}^m \EE_{\substack{\tilde \xb, \tilde \xb',\tilde \xb'' \sim \Dcal_{\tilde x}, \\ \alpha, \alpha', \alpha''\sim \Dcal_{\alpha}}} [\ell_\cont(\xb^{+}_{i},\xb^{++}_{i},\xb^{-}_{j}) ]
\\ &=\frac{1}{n^{2}} \sum_{i=1}^m \sum_{j\in S_{y_i}}   \EE_{\substack{\tilde \xb, \tilde \xb',\tilde \xb'' \sim \Dcal_{\tilde x}, \\ \alpha, \alpha', \alpha''\sim \Dcal_{\alpha}}} \left[  \ell_{\class}\left(f(\xb^+_{i})   , y_{i} \right) \right]+ \frac{1}{n^2}\sum_{i=1}^n\left[(n-| S_{y_i}|)\Ecal_{y}\right],
\end{align*}
where $\xb^{+}_i =\xb_{i}+\alpha\delta(\xb_{i},\tilde \xb)$,  $\xb^{++} _{i}= \xb _{i}+ \alpha'\delta(\xb_{i},\tilde \xb')$,  $\xb^{-} _{j}=\bar \xb_{j} + \alpha''\delta(\bar \xb_{j},\tilde \xb'')$, $S_y=\{i \in [m]  : y_i \neq y\}$, $f(\xb^+_{i})=\|(h(\xb^{+} _{i})\|^{-1}h(\xb^{+} _{i})\T   \tilde w$, and $[m]=\{1,\dots,m \}$. Here, we used the fact that $\bar \rho({y})=\frac{|S_y|}{n}$ where $|S_y|$ is the number of elements in the set $S_y$. In general,   in Theorem \ref{thm:1}, we can set the distribution $\Dcal$ to take into account additional data augmentations (that generate infinite number of samples) and   the different ways that  we generate positive and negative pairs.

\paragraph{On the interpretation of Theorem \ref{thm:2} for deep neural networks.} Consider the case of deep neural networks with ReLU in the form of  $f(\xb)=W^{(H)}\sigma^{(H-1)}(W^{(H-1)}\sigma^{(H-2)}(\cdots\sigma^{(1)}(W^{(1)}\xb)\cdots))$, where $W^{(l)}$ is the weight matrix and  $\sigma^{(l)}$ is the ReLU nonlinear function at the $l$-th layer. In this case, we have 
$$
\|\nabla f(\xb)\|=\|W^{(H)} \dot\sigma^{(H-1)}W^{(H-1)}\dot\sigma^{(H-2)}\cdots\dot\sigma^{(1)}W^{(1)}\|,
$$ 
where $\dot\sigma^{(l)}= \frac{\partial \sigma^{(l)}(q)}{\partial q }\vert_{q=W^{(l-1)}\sigma^{(l-2)}(\cdots\sigma^{(1)}(W^{(1)}x)\cdots)}$ is a Jacobian matrix and hence $W^{(H)} \dot\sigma^{(H-1)}W^{(H-1)}\dot\sigma^{(H-2)}\cdots \allowbreak \dot\sigma^{(1)}W^{(1)}$ is the sum of the product of path weights. Thus,  regularizing$\|\nabla f(\xb)\|$ tends to promote generalization as it corresponds to the path weight norm   used  in generalization error bounds  in  previous work \citep{kawaguchi2017generalization}. 

\section{Proof} \label{app:proof}

In this section, we present complete proofs for our theoretical results. We note that in the proofs and in theorems, the distribution $\Dcal$ is arbitrary. As an simplest example of the practical setting, we can set $\Dcal$ to represent the empirical measure $\frac{1}{n}\sum_{i=1}^m \delta_{(x_i,y_i) }$ for training data $((x_i,y_i))_{i=1}^m$ (where the Dirac measures $\delta_{(x_i,y_i)}$), which yields  the following:
\begin{align} \label{eq:new:1}
& \EE_{\hspace{-2pt}\substack{\xb, \bar \xb\sim \Dcal_{x}, \\ \tilde \xb, \tilde \xb',\tilde \xb'' \sim \Dcal_{\tilde x}, \\ \alpha, \alpha', \alpha''\sim \Dcal_{\alpha}}} [\ell_\cont(\xb^{+},\xb^{++},\xb^{-}) ]
= \frac{1}{n^{2}} \sum_{i=1}^m \sum_{j=1}^m \EE_{\substack{\tilde \xb, \tilde \xb',\tilde \xb'' \sim \Dcal_{\tilde x}, \\ \alpha, \alpha', \alpha''\sim \Dcal_{\alpha}}} [\ell_\cont(\xb^{+}_{i},\xb^{++}_{i},\xb^{-}_{j}) ],
\end{align}
where $\xb^{+}_i =\xb_{i}+\alpha\delta(\xb_{i},\tilde \xb)$,  $\xb^{++} _{i}= \xb _{i}+ \alpha'\delta(\xb_{i},\tilde \xb')$, and $\xb^{-} _{j}=\bar \xb_{j} + \alpha''\delta(\bar \xb_{j},\tilde \xb'')$.
In equation \eqref{eq:new:1}, we can more easily see that for each single point $x_i$, we have the $m$ negative examples as: 
$$
\sum_{j=1}^m \EE_{\substack{\tilde \xb, \tilde \xb',\tilde \xb'' \sim \Dcal_{\tilde x}, \\ \alpha, \alpha', \alpha''\sim \Dcal_{\alpha}}} [\ell_\cont(\xb^{+}_{i},\xb^{++}_{i},\xb^{-}_{j}) ].
$$
Thus, for each single point $x_i$,  all points generated based on all other points $\bar x_j$ for $j=1,\dots,m$ are treated as negatives, whereas the positives are the ones generated based on the particular point $x_i$. The ratio of negatives increases as the number of original data points increases and our proofs apply for any number of original data points.

\subsection{Proof of Theorem \ref{thm:1}}
We begin by introducing additional notation to be used in our proof. For two vectors $\bm q$ and $\bm q'$, define 
$$
\overline \cov[\bm q, \bm q'] = \sum_k \cov(\bm q_{k}, \bm q_{k}') 
$$
Let $\rho_{y}=\EE_{\bar y|y}[\one{\bar y=y}]=\sum_{\bar y \in \{0,1\}}p_{\bar y }(\bar y\mid y)\one{\bar y =y}=\Pr(\bar y =y \mid y)$. For the completeness, we first recall the following well known fact:

\begin{lemma} \label{lemma:1}
For any $y \in \{0,1\}$ and $q\in \RR$,
$$
\ell(q, y)= - \log \left(\frac{\exp( yq )}{1+\exp(q )} \right)
$$
\end{lemma}
\begin{proof} By simple arithmetic manipulations,  
\begin{align*}
\ell(q, y) &= - y \log \left(\frac{1}{1+\exp(-q )} \right) - (1-y) \log\left(1-\frac{1}{1+\exp(-q )} \right)
\\ & = - y \log \left(\frac{1}{1+\exp(-q )} \right) - (1-y) \log\left(\frac{\exp(-q )}{1+\exp(-q )} \right)
\\ & = - y \log \left(\frac{\exp(q )}{1+\exp(q )} \right) - (1-y) \log\left(\frac{1}{1+\exp(q )} \right)
\\ & = 
\begin{cases}-\log \left(\frac{\exp(q )}{1+\exp(q )} \right)  & \text{if } y =1 \\
-\log\left(\frac{1}{1+\exp(q )} \right) &  \text{if } y =0  \\
\end{cases}
\\ & = - \log \left(\frac{\exp( y q )}{1+\exp(q )} \right). 
\end{align*}
\end{proof}

Before starting the main parts of the proof, we also prepare  the following simple facts: \begin{lemma} \label{lemma:2}
For any $(\xb^{+},\xb^{++}, \xb^-)$, we have
$$
\ell_\cont(\xb^{+}_{}, \xb^{++}_{}, \xb^{-})=\ell(\simi[h(\xb^{+}_{}), h(\xb^{++}_{})]-\simi[h(\xb^{+}_{}), h(\xb^{-})], 1)
$$
\end{lemma}
\begin{proof}
By simple arithmetic manipulations,
\begin{align*}
\ell_\cont(\xb^{+}, \xb^{++}, \xb^{-}) &=- \log \frac{\exp( \simi[h(\xb^{+}_{}), h(\xb^{++}_{})])}{\exp(\simi[ h(\xb^{+}), h(\xb^{++})])+\exp(\simi[h(\xb^+), h(\xb^{-})] )}
\\ & =- \log \frac{1}{1+\exp(\simi[h(\xb^+), h(\xb^{-})]- \simi[h(\xb^{+}_{}), h(\xb^{++}_{})] )}
\\ & =- \log \frac{\exp( \simi[h(\xb^{+}), h(\xb^{++}_{})]-\simi[h(\xb^{+}_{}), h(\xb^{-})] )}{1+\exp( \simi[h(\xb^{+}_{}), h(\xb^{++}_{})]-\simi[h(\xb^{+}_{}), h(\xb^{-})] )}
\end{align*}
Using Lemma \ref{lemma:1} with $q= \simi[h(\xb^{+}_{}), h(\xb^{++})]-\simi[h(\xb^{+}_{}), h(\xb^{-})]$, this yields the desired statement. 
\end{proof}
\begin{lemma} \label{lemma:3}
For any $y \in \{0,1\}$ and $q\in \RR$,
$$
\ell(-q, 1)=\ell(q, 0).
$$
\end{lemma}
\begin{proof}
Using Lemma \ref{lemma:1}, 
$$
\ell(-q, 1)=- \log \left(\frac{\exp( -q )}{1+\exp(-q )} \right)=- \log \left(\frac{1}{1+\exp(q )} \right) =\ell(q, 0).
$$
\end{proof}

With these facts, we are now ready to start our proof. We first prove the relationship between the contrastive loss and classification loss under an ideal situation: 

\begin{lemma} \label{lemma:4}
Assume that $\xb^{+} = \xb + \alpha\delta(\xb,\tilde \xb)$, $\xb^{++} = \xb + \alpha'\delta(\xb,\tilde \xb')$, $\xb^{-} =\bar \xb + \alpha''\delta(\bar \xb,\tilde \xb'')$, and  $\simi[z, z']=\frac{z\T z'}{\zeta(z)\zeta(z')}$ where $\zeta: z \mapsto \zeta(z) \in \RR$. Then for any $(\alpha,\tilde \xb,\delta,\zeta )$ and $(y,\bar y)$ such that  $y \neq \bar y$,
we have that\begin{align*}
\EE_{\xb \sim \Dcal_y} \EE_{\bar \xb\sim \Dcal_{\bar y \neq y}} \EE_{\substack{\tilde \xb',\tilde \xb'' \sim \Dcal_{\tilde x}, \\ \alpha', \alpha''\sim \Dcal_{\alpha}}}[\ell_\cont(\xb^{+},\xb^{++},\xb^{-}) ]=\EE_{\xb \sim \Dcal_y} \EE_{\bar \xb\sim \Dcal_{\bar y \neq y}} \EE_{\substack{\tilde \xb',\tilde \xb'' \sim \Dcal_{\tilde x}, \\ \alpha', \alpha''\sim \Dcal_{\alpha}}} \left[ \ell\left(\frac{h(\xb^{+} )\T \tilde w}{\zeta(h(\xb^{+} ))}   , y \right) \right],
\end{align*}
\end{lemma}
\begin{proof}
Using Lemma \ref{lemma:2} and the assumption on $\simi$, 
\begin{align*}
\ell_\cont(\xb^{+},\xb^{++},\xb^{-}) &=\ell(\simi[h(\xb^+), h(\xb^{++})] - \simi[h(\xb^+), h(\xb^{-})], 1)
\\ & =\ell\left( \frac{h(\xb^+)\T h(\xb^{++})}{\zeta(h(\xb^+))\zeta(h(\xb^{++}))}- \frac{h(\xb^+)\T h(\xb^-)}{\zeta(h(\xb^+))\zeta(h(\xb^-))}, 1 \right)
\\ & = \ell\left( \frac{h(\xb^+)\T }{\zeta(h(\xb^+))} \left(\frac{h(\xb^{++})}{\zeta(h(\xb^{++}))}- \frac{h(\xb^-)}{\zeta(h(\xb^-))} \right), 1 \right).
\end{align*}
Therefore,
\begin{align*}
&\EE_{\xb \sim \Dcal_y} \EE_{\bar \xb\sim \Dcal_{\bar y \neq y}} \EE_{\substack{\xb',\tilde \xb''\sim \Dcal_{\tilde x}, \\ \alpha', \alpha''\sim \Dcal_{\alpha}}}[\ell_\cont(\xb^{+},\xb^{++},\xb^{-}) ]
\\ &=\scalebox{0.9}{$\displaystyle \EE_{\substack{\xb \sim \Dcal_y, \\ \bar \xb\sim \Dcal_{\bar y \neq y}}} \EE_{\substack{\tilde \xb',\tilde \xb'', \\ \alpha', \alpha''}} \left[\ell\left( \frac{h(\xb+\alpha\delta(\xb,\tilde \xb))\T }{\zeta(h(\xb+\alpha\delta(\xb,\tilde \xb)))} \left(\frac{h( \xb + \alpha'\delta(\xb,\tilde \xb'))}{\zeta(h( \xb + \alpha'\delta(\xb,\tilde \xb')))}- \frac{h(\bar \xb + \alpha''\delta(\bar \xb,\tilde \xb''))}{\zeta(h(\bar \xb  + \alpha''\delta(\bar \xb,\tilde \xb'')))} \right), 1 \right) \right] $}
\\ & = \scalebox{0.9}{$\displaystyle  \begin{cases} \EE_{\substack{\xb^{1} \sim \Dcal_1,\\\xb^0 \sim \Dcal_{0} }}  \EE_{\substack{\tilde \xb',\tilde \xb'' \\ \alpha', \alpha''}} \left[\ell\left( \frac{h(\xb^{1}+\alpha\delta(\xb^{1},\tilde \xb))\T }{\zeta(h(\xb^{1}+\alpha\delta(\xb^{1},\tilde \xb)))} \left(\frac{h(\xb^{1} + \alpha'\delta(\xb^{1},\tilde \xb'))}{\zeta(h(\xb^{1} + \alpha'\delta(\xb^{1},\tilde \xb')))}- \frac{h(\xb^0+ \alpha''\delta(\xb^0,\tilde \xb''))}{\zeta(h(\xb^0+ \alpha''\delta(\xb^0,\tilde \xb'')))} \right), 1 \right) \right] & \text{if } y =1 \\
\EE_{\substack{\xb^{0} \sim \Dcal_0,\\\xb^1 \sim \Dcal_{1} }} \EE_{\substack{\tilde \xb',\tilde \xb'', \\ \alpha', \alpha''}} \left[\ell\left( \frac{h(\xb^{0}+\alpha\delta(\xb^{0},\tilde \xb))\T }{\zeta(h(\xb^{0}+\alpha\delta(\xb^{0},\tilde \xb)))} \left(\frac{h(\xb^{0} + \alpha'\delta(\xb^{0},\tilde \xb'))}{\zeta(h(\xb^{0} + \alpha'\delta(\xb^{0},\tilde \xb')))}- \frac{h(\xb^1+ \alpha''\delta(\xb^1,\tilde \xb''))}{\zeta(h(\xb^1+ \alpha''\delta(\xb^1,\tilde \xb'')))} \right), 1 \right) \right] & \text{if } y =0 \\
\end{cases} $}
\\ & = \scalebox{0.9}{$\displaystyle \begin{cases} \EE_{\substack{\xb^{1} \sim \Dcal_1,\\\xb^0 \sim \Dcal_{0} }} \EE_{\substack{\tilde \xb',\tilde \xb'', \\ \alpha', \alpha''}} \left[\ell\left( \frac{h(\xb^{1}+\alpha\delta(\xb^{1},\tilde \xb))\T }{\zeta(h(\xb^{1}+\alpha\delta(\xb^{1},\tilde \xb)))} \left(\frac{h(\xb^{1} + \alpha'\delta(\xb^{1},\tilde \xb'))}{\zeta(h(\xb^{1} + \alpha'\delta(\xb^{1},\tilde \xb')))}- \frac{h(\xb^0+ \alpha''\delta(\xb^0,\tilde \xb''))}{\zeta(h(\xb^0+ \alpha''\delta(\xb^0,\tilde \xb'')))} \right), 1 \right) \right] & \text{if } y =1 \\
\EE_{\substack{\xb^{0} \sim \Dcal_0,\\\xb^1 \sim \Dcal_{1} }} \EE_{\substack{\tilde \xb',\tilde \xb'', \\ \alpha', \alpha''}} \left[\ell\left( \frac{h(\xb^{0}+\alpha\delta(\xb^{0},\tilde \xb))\T }{\zeta(h(\xb^{0}+\alpha\delta(\xb^{0},\tilde \xb)))} \left(\frac{h(\xb^{0} + \alpha''\delta(\xb^{0},\tilde \xb''))}{\zeta(h(\xb^{0} + \alpha''\delta(\xb^{0},\tilde \xb'')))}- \frac{h(\xb^1+ \alpha'\delta(\xb^1,\tilde \xb'))}{\zeta(h(\xb^1+ \alpha'\delta(\xb^1,\tilde \xb')))} \right), 1 \right) \right] & \text{if } y =0 \\
\end{cases} $}
\\ & = \begin{cases} \EE_{\xb^{1} \sim \Dcal_1} \EE_{\xb^0 \sim \Dcal_{0}} \EE_{\substack{\tilde \xb',\tilde \xb''\sim \Dcal_{\tilde x}, \\ \alpha', \alpha''\sim \Dcal_{\alpha}}}\left[\ell\left( \frac{h(\xb^{1}+\alpha\delta(\xb^{1},\tilde \xb))\T }{\zeta(h(\xb^{1}+\alpha\delta(\xb^{1},\tilde \xb)))} \widetilde W(\xb^{1}, \xb^0 ), 1 \right) \right] & \text{if } y =1 \\
\EE_{\xb^{0} \sim \Dcal_0} \EE_{\xb^1 \sim \Dcal_{1}}\EE_{\substack{\tilde \xb',\tilde \xb''\sim \Dcal_{\tilde x}, \\ \alpha', \alpha''\sim \Dcal_{\alpha}}} \left[\ell\left( -\frac{h(\xb^{0}+\alpha\delta(\xb^{0},\tilde \xb))\T }{\zeta(h(\xb^{0}+\alpha\delta(\xb^{0},\tilde \xb)))}  \widetilde W(\xb^{1}, \xb^0 ), 1 \right) \right] & \text{if } y =0 \\
\end{cases} 
\end{align*}
where 
$$
\widetilde W(\xb^{1}, \xb^0 )=\frac{h(\xb^{1} + \alpha'\delta(\xb^{1},\tilde \xb'))}{\zeta(h(\xb^{1} + \alpha'\delta(\xb^{1},\tilde \xb')))}- \frac{h(\xb^0+ \alpha''\delta(\xb^0,\tilde \xb''))}{\zeta(h(\xb^0+ \alpha''\delta(\xb^0,\tilde \xb'')))}.
$$
Using Lemma \ref{lemma:3}, 
\begin{align*}
&\EE_{\xb \sim \Dcal_y} \EE_{\bar \xb\sim \Dcal_{\bar y \neq y}} \EE_{\substack{\xb',\tilde \xb''\sim \Dcal_{\tilde x}, \\ \alpha', \alpha''\sim \Dcal_{\alpha}}} [\ell_\cont(\xb^{+},\xb^{++},\xb^{-}) ]
\\ & = \begin{cases} \EE_{\xb^{1} \sim \Dcal_1} \EE_{\xb^0 \sim \Dcal_{0}} \EE_{\substack{\tilde \xb',\tilde \xb''\sim \Dcal_{\tilde x}, \\ \alpha', \alpha''\sim \Dcal_{\alpha}}}\left[\ell\left( \frac{h(\xb^{1}+\alpha\delta(\xb^{1},\tilde \xb))\T }{\zeta(h(\xb^{1}+\alpha\delta(\xb^{1},\tilde \xb)))} \widetilde W(\xb^{1}, \xb^0 ), 1 \right) \right] & \text{if } y =1 \\
\EE_{\xb^{0} \sim \Dcal_0} \EE_{\xb^1 \sim \Dcal_{1}} \EE_{\substack{\tilde \xb',\tilde \xb''\sim \Dcal_{\tilde x}, \\ \alpha', \alpha''\sim \Dcal_{\alpha}}} \left[\ell\left( \frac{h(\xb^{0}+\alpha\delta(\xb^{0},\tilde \xb))\T }{\zeta(h(\xb^{0}+\alpha\delta(\xb^{0},\tilde \xb)))}  \widetilde W(\xb^{1}, \xb^0 ), 0 \right) \right] & \text{if } y =0 \\
\end{cases} 
\\ & = \begin{cases} \EE_{\xb^{1} \sim \Dcal_1} \EE_{\xb^0 \sim \Dcal_{0}}\EE_{\substack{\tilde \xb',\tilde \xb''\sim \Dcal_{\tilde x}, \\ \alpha', \alpha''\sim \Dcal_{\alpha}}} \left[\ell\left( \frac{h(\xb^{1}+\alpha\delta(\xb^{1},\tilde \xb))\T }{\zeta(h(\xb^{1}+\alpha\delta(\xb^{1},\tilde \xb)))} \widetilde W(\xb^{1}, \xb^0 ), y \right) \right] & \text{if } y =1 \\
\EE_{\xb^{0} \sim \Dcal_0} \EE_{\xb^1 \sim \Dcal_{1}}\EE_{\substack{\tilde \xb',\tilde \xb''\sim \Dcal_{\tilde x}, \\ \alpha', \alpha''\sim \Dcal_{\alpha}}} \left[\ell\left( \frac{h(\xb^{0}+\alpha\delta(\xb^{0},\tilde \xb))\T }{\zeta(h(\xb^{0}+\alpha\delta(\xb^{0},\tilde \xb)))}  \widetilde W(\xb^{1}, \xb^0 ), y \right) \right] & \text{if } y =0 \\
\end{cases} 
\\ & = \EE_{\xb \sim \Dcal_y} \EE_{\bar \xb\sim \Dcal_{\bar y \neq y}} \EE_{\substack{\xb',\tilde \xb''\sim \Dcal_{\tilde x}, \\ \alpha', \alpha''\sim \Dcal_{\alpha}}} \left[\ell\left( \frac{h(\xb+\alpha\delta(\xb,\tilde \xb))\T }{\zeta(h(\xb+\alpha\delta(\xb,\tilde \xb)))} \tilde w, y \right)  \right] 
\end{align*}

\end{proof}

Using the above  the relationship under the ideal situation, we now proves the relationship under the practical situation:

\begin{lemma} \label{lemma:5}
Assume that $\xb^{+} = \xb + \alpha\delta(\xb,\tilde \xb)$, $\xb^{++} = \xb + \alpha'\delta(\xb,\tilde \xb')$, $\xb^{-} =\bar \xb + \alpha''\delta(\bar \xb,\tilde \xb'')$, and  $\simi[z, z']=\frac{z\T z'}{\zeta(z)\zeta(z')}$ where $\zeta: z \mapsto \zeta(z) \in \RR$. Then for any $(\alpha,\tilde \xb,\delta,\zeta,y )$,
we have that
\begin{align*}
&\EE_{\bar y|y}\EE_{\substack{\xb \sim \Dcal_y, \\ \bar \xb\sim \Dcal_{\bar y } }}  \EE_{\substack{\tilde \xb',\tilde \xb'' \sim \Dcal_{\tilde x}, \\ \alpha', \alpha''\sim \Dcal_{\alpha}}} [\ell_\cont(\xb^{+},\xb^{++},\xb^{-}) ]
\\ &=(1 - \rho_{y})\EE_{\substack{\xb \sim \Dcal_y, \\ \bar \xb\sim \Dcal_{\bar y \neq y} }}  \EE_{\substack{\tilde \xb',\tilde \xb'' \sim \Dcal_{\tilde x}, \\ \alpha', \alpha''\sim \Dcal_{\alpha}}} \left[ \ell\left(\frac{h(\xb^{+} )\T \tilde w}{\zeta(h(\xb^{+} ))}   , y \right) \right]+\rho_{y} E
\end{align*}
where 
\begin{align*}
E&=\EE_{\xb, \bar \xb\sim \Dcal_y} \EE_{\substack{\tilde \xb',\tilde \xb'' \sim \Dcal_{\tilde x}, \\ \alpha', \alpha''\sim \Dcal_{\alpha}}} \left[ \log\left(1+\exp \left[-\frac{h(\xb^+)\T }{\zeta(h(\xb^+))} \left(\frac{h(\xb^{++})}{\zeta(h(\xb^{++}))}- \frac{h(\xb^-)}{\zeta(h(\xb^-))} \right)\right] \right) \right]
\\ &\ge  \log \left(1+\exp \left[-  \overline \cov_{\substack{\xb\sim \Dcal_y, \\ \tilde \xb'\sim \Dcal_{\tilde x}, \\ \alpha'\sim \Dcal_{\alpha}}}\left[\frac{h(\xb^+) }{\zeta(h(\xb^+))},\frac{h(\xb^{++})}{\zeta(h(\xb^{++}))} \right] \right] \right)  
\end{align*}

\end{lemma}
\begin{proof}
Using Lemma \ref{lemma:4},
\begin{align*}
&\EE_{\bar y|y}\EE_{\substack{\xb \sim \Dcal_y, \\ \bar \xb\sim \Dcal_{\bar y } }}  \EE_{\substack{\tilde \xb',\tilde \xb'' \sim \Dcal_{\tilde x}, \\ \alpha', \alpha''\sim \Dcal_{\alpha}}} [\ell_\cont(\xb^{+},\xb^{++},\xb^{-})]
\\ & =\sum_{\bar y\in\{0,1\}}p_{\bar y }(\bar y \mid y)\EE_{\substack{\xb \sim \Dcal_y, \\ \bar \xb\sim \Dcal_{\bar y } }}  \EE_{\substack{\tilde \xb',\tilde \xb'' \sim \Dcal_{\tilde x}, \\ \alpha', \alpha''\sim \Dcal_{\alpha}}} [\ell_\cont(\xb^{+},\xb^{++},\xb^{-})]
\\ & = \scalebox{0.9}{$\displaystyle \Pr(\bar y =0 \mid y)\EE_{\substack{\xb \sim \Dcal_y,\\ \bar \xb \sim \Dcal_{\bar y =0,} \\ \tilde \xb',\tilde \xb'' \sim \Dcal_{\tilde x}, \\ \alpha', \alpha''\sim \Dcal_{\alpha}}}[\ell_\cont(\xb^{+},\xb^{++},\xb^{-})]+ \Pr(\bar y =1 \mid y)\EE_{\substack{\xb \sim \Dcal_y,\\ \bar \xb \sim \Dcal_{\bar y =1,} \\ \tilde \xb',\tilde \xb'' \sim \Dcal_{\tilde x}, \\ \alpha', \alpha''\sim \Dcal_{\alpha}}}[\ell_\cont(\xb^{+},\xb^{++},\xb^{-})] $}
\\ & =\scalebox{0.9}{$\displaystyle \Pr(\bar y \neq y \mid y)  \EE_{\substack{\xb \sim \Dcal_y,\\ \bar \xb \sim \Dcal_{\bar y\neq y} \\ ,\tilde \xb',\tilde \xb'' \sim \Dcal_{\tilde x}, \\ \alpha', \alpha''\sim \Dcal_{\alpha}}}[\ell_\cont(\xb^{+},\xb^{++},\xb^{-})]+ \Pr(\bar y =y \mid y)_{\substack{}}  \EE_{\substack{\xb \sim \Dcal_y,\\ \bar \xb \sim \Dcal_{\bar y =y}, \\ \tilde \xb',\tilde \xb'' \sim \Dcal_{\tilde x}, \\ \alpha', \alpha''\sim \Dcal_{\alpha}}}[\ell_\cont(\xb^{+},\xb^{++},\xb^{-})] $}
\\ & = \scalebox{0.9}{$\displaystyle (1-\rho_{y})\EE_{\substack{\xb \sim \Dcal_y,\\ \bar \xb \sim \Dcal_{\bar y\neq y}}} \EE_{\substack{\tilde \xb',\tilde \xb'' \sim \Dcal_{\tilde x}, \\ \alpha', \alpha''\sim \Dcal_{\alpha}}}[\ell_\cont(\xb^{+},\xb^{++},\xb^{-})]+\rho_{y}\EE_{\substack{\xb \sim \Dcal_y,\\ \bar \xb \sim \Dcal_{\bar y =y}}} \EE_{\substack{\tilde \xb',\tilde \xb'' \sim \Dcal_{\tilde x}, \\ \alpha', \alpha''\sim \Dcal_{\alpha}}}[\ell_\cont(\xb^{+},\xb^{++},\xb^{-})] $}
\\ & =\scalebox{0.9}{$\displaystyle (1-\rho_{y})\EE_{\substack{\xb \sim \Dcal_y,\\ \bar \xb \sim \Dcal_{\bar y\neq y}}}\EE_{\substack{\tilde \xb',\tilde \xb'' \sim \Dcal_{\tilde x}, \\ \alpha', \alpha''\sim \Dcal_{\alpha}}} \left[ \ell\left(\frac{h(\xb^{+} )\T \tilde w}{\zeta(h(\xb^{+} ))}   , y \right) \right]+\rho_{y}\EE_{\substack{\xb \sim \Dcal_y,\\ \bar \xb \sim \Dcal_{\bar y =y}}}\EE_{\substack{\tilde \xb',\tilde \xb'' \sim \Dcal_{\tilde x}, \\ \alpha', \alpha''\sim \Dcal_{\alpha}}}[\ell_\cont(\xb^{+},\xb^{++},\xb^{-})],  $}
\end{align*}
which obtain the desired statement for the first term. We 
now focus on the second term. Using Lemmas \ref{lemma:1} and \ref{lemma:2}, with $q=\frac{h(\xb^+)\T }{\zeta(h(\xb^+))} \left(\frac{h(\xb^{++})}{\zeta(h(\xb^{++}))}- \frac{h(\xb^-)}{\zeta(h(\xb^-))} \right)$, 
\begin{align*}
\ell_\cont(\xb^{+},\xb^{++},\xb^{-}) =\ell\left( q, 1 \right) =- \log \left(\frac{\exp(q )}{1+\exp(q )} \right) =- \log \left(\frac{1}{1+\exp(-q )} \right)  =\log \left(1+\exp(-q )\right).  
\end{align*}
Therefore,

\begin{align*}
&\EE_{\xb \sim \Dcal_y} \EE_{\bar \xb \sim \Dcal_{\bar y =y}} \EE_{\substack{\tilde \xb',\tilde \xb'' \sim \Dcal_{\tilde x}, \\ \alpha', \alpha''\sim \Dcal_{\alpha}}}[\ell_\cont(\xb^{+},\xb^{++},\xb^{-})] \\ &=\EE_{\xb,\bar \xb\sim \Dcal_y}\EE_{\substack{\tilde \xb',\tilde \xb'' \sim \Dcal_{\tilde x}, \\ \alpha', \alpha''\sim \Dcal_{\alpha}}}  \left[\log \left(1+\exp \left[-\frac{h(\xb^+)\T }{\zeta(h(\xb^+))} \left(\frac{h(\xb^{++})}{\zeta(h(\xb^{++}))}- \frac{h(\xb^-)}{\zeta(h(\xb^-))} \right) \right] \right)\right] =E,
\end{align*}
which proves the desired statement with $E$. We now focus on the  lower bound  on $E$.  By using the convexity of $q \mapsto \log \left(1+\exp(-q )\right)$ and Jensen's inequality,
\begin{align*}
E &\ge  \log \left(1+\exp \left[\EE_{\xb,\bar \xb} \EE_{\substack{\tilde \xb',\tilde \xb'' , \\ \alpha', \alpha''}} \left[\frac{h(\xb^+)\T }{\zeta(h(\xb^+))} \left( \frac{h(\xb^-)}{\zeta(h(\xb^-))}-\frac{h(\xb^{++})}{\zeta(h(\xb^{++}))} \right) \right] \right] \right) 
\\ & =  \log \left(1+\exp \left[\EE_{} \left[\frac{h(\xb^+)\T }{\zeta(h(\xb^+))}\frac{h(\xb^-)}{\zeta(h(\xb^-))} \right]-_{}\EE_{\substack{}}\left[\frac{h(\xb^+)\T }{\zeta(h(\xb^+))}\frac{h(\xb^{++})}{\zeta(h(\xb^{++}))}  \right] \right] \right)
\\ & =  \log \left(1+\exp \left[\EE_{\substack{}} \left[\frac{h(\xb^+)\T }{\zeta(h(\xb^+))}  \right] \EE\left[\frac{h(\xb^-)}{\zeta(h(\xb^-))} \right]-\EE_{\substack{}}\left[\frac{h(\xb^+)\T }{\zeta(h(\xb^+))} \frac{h(\xb^{++})}{\zeta(h(\xb^{++}))}  \right] \right] \right)
\end{align*}
Here, we have
\begin{align*}
&\EE_{\substack{\xb\sim \Dcal_y, \\ \tilde \xb'\sim \Dcal_{\tilde x}, \\ \alpha'\sim \Dcal_{\alpha}}}\left[\frac{h(\xb^+)\T }{\zeta(h(\xb^+))} \frac{h(\xb^{++})}{\zeta(h(\xb^{++}))}  \right]
 \\ &=\EE_{\substack{\xb\sim \Dcal_y, \\ \tilde \xb'\sim \Dcal_{\tilde x}, \\ \alpha'\sim \Dcal_{\alpha}}} \sum_k \left(\frac{h(\xb^+)}{\zeta(h(\xb^+))}\right)_k \left(\frac{h(\xb^{++})}{\zeta(h(\xb^{++}))}\right)_k 
\\ & = \sum_k \EE_{\substack{\xb\sim \Dcal_y, \\ \tilde \xb'\sim \Dcal_{\tilde x}, \\ \alpha'\sim \Dcal_{\alpha}}}\left(\frac{h(\xb^+)}{\zeta(h(\xb^+))}\right)_k \left(\frac{h(\xb^{++})}{\zeta(h(\xb^{++}))}\right)_k 
\\ & = \scalebox{0.9}{$\displaystyle \sum_k \EE_{}\left[\left(\frac{h(\xb^+)}{\zeta(h(\xb^+))}\right)_k \right] \EE_{\substack{}}\left[ \left(\frac{h(\xb^{++})}{\zeta(h(\xb^{++}))}\right)_k\right] +  \sum_k\cov_{\substack{}}\left(\left(\frac{h(\xb^+)}{\zeta(h(\xb^+))}\right)_k, \left(\frac{h(\xb^{++})}{\zeta(h(\xb^{++}))}\right)_k \right) $}
\\ & =\EE_{\xb\sim \Dcal_y}\left[\frac{h(\xb^+)\T}{\zeta(h(\xb^+))} \right] \EE_{\substack{\xb\sim \Dcal_y, \\ \tilde \xb'\sim \Dcal_{\tilde x}, \\ \alpha'\sim \Dcal_{\alpha}}}\left[ \frac{h(\xb^{++})}{\zeta(h(\xb^{++}))}\right] +  \overline \cov_{\substack{}}\left[\frac{h(\xb^+) }{\zeta(h(\xb^+))},\frac{h(\xb)}{\zeta(h(\xb))} \right]
\end{align*}
Since $\EE_{\substack{\xb\sim \Dcal_y, \\ \tilde \xb'\sim \Dcal_{\tilde x}, \\ \alpha'\sim \Dcal_{\alpha}}}\left[ \left(\frac{h(\xb^{++})}{\zeta(h(\xb^{++}))}\right)\right] =\EE_{\substack{\bar \xb\sim \Dcal_y, \\ \tilde \xb''\sim \Dcal_{\tilde x}, \\ \alpha''\sim \Dcal_{\alpha}}}\left[\frac{h(\xb^-)}{\zeta(h(\xb^-))} \right]$,
\begin{align*}
&\EE_{\substack{\xb\sim \Dcal_y}} \left[\frac{h(\xb^+)\T }{\zeta(h(\xb^+))}  \right] \EE_{\substack{\bar \xb\sim \Dcal_y, \\ \tilde \xb''\sim \Dcal_{\tilde x}, \\ \alpha''\sim \Dcal_{\alpha}}}\left[\frac{h(\xb^-)}{\zeta(h(\xb^-))} \right]-\EE_{\substack{\xb\sim \Dcal_y, \\ \tilde \xb'\sim \Dcal_{\tilde x}, \\ \alpha'\sim \Dcal_{\alpha}}}\left[\frac{h(\xb^+)\T }{\zeta(h(\xb^+))} \frac{h(\xb^{++})}{\zeta(h(\xb^{++}))}  \right] 
\\ & = \scalebox{0.8}{$\displaystyle  \EE_{} \left[\frac{h(\xb^+)\T }{\zeta(h(\xb^+))}  \right] \left(\EE_{\substack{\bar \xb\sim \Dcal_y, \\ \tilde \xb''\sim \Dcal_{\tilde x}, \\ \alpha''\sim \Dcal_{\alpha}}}\left[\frac{h(\xb^-)}{\zeta(h(\xb^-))} \right]-\EE_{\substack{\xb\sim \Dcal_y, \\ \tilde \xb'\sim \Dcal_{\tilde x}, \\ \alpha'\sim \Dcal_{\alpha}}}\left[ \frac{h(\xb^{++})}{\zeta(h(\xb^{++}))}\right]  \right)-  \overline \cov_{\substack{}}\left[\frac{h(\xb^+) }{\zeta(h(\xb^+))},\frac{h(\xb^{++})}{\zeta(h(\xb^{++}))} \right] $}
\\ & =-  \overline \cov\left[\frac{h(\xb^+) }{\zeta(h(\xb^+))},\frac{h(\xb^{++})}{\zeta(h(\xb^{++}))} \right]
\end{align*}
Substituting this to the above inequality on $E$, 
$$
E\ge  \log \left(1+\exp \left[-  \overline \cov\left[\frac{h(\xb^+) }{\zeta(h(\xb^+))},\frac{h(\xb^{++})}{\zeta(h(\xb^{++}))} \right] \right] \right) , 
$$ 
which proves the desired statement for the lower bound on $E$.
\end{proof}

With these lemmas, we are now ready to prove Theorem \ref{thm:1}:

\begin{proof}[Proof of Theorem \ref{thm:1}]
From Lemma \ref{lemma:5}, we have that 
\begin{align*}
&\EE_{\bar y|y}\EE_{\substack{\xb \sim \Dcal_y, \\ \bar \xb\sim \Dcal_{\bar y } }}  \EE_{\substack{\tilde \xb',\tilde \xb'' \sim \Dcal_{\tilde x}, \\ \alpha', \alpha''\sim \Dcal_{\alpha}}} [\ell_\cont(\xb^{+},\xb^{++},\xb^{-}) ]
\\ &=(1 - \rho_{y})\EE_{\substack{\xb \sim \Dcal_y, \\ \bar \xb\sim \Dcal_{\bar y } }}  \EE_{\substack{\tilde \xb',\tilde \xb'' \sim \Dcal_{\tilde x}, \\ \alpha', \alpha''\sim \Dcal_{\alpha}}} \left[ \ell_{\class}\left(\frac{h(\xb^{+} )\T \tilde w}{\zeta(h(\xb^{+} ))}   , y \right) \right]+\rho_{y} E
\end{align*}
By taking expectation over $y$ in both sides,
\begin{align*}
&\EE_{y,\bar y}\EE_{\substack{\xb \sim \Dcal_{y}, \\ \bar \xb\sim \Dcal_{\bar y } }}  \EE_{\substack{\tilde \xb',\tilde \xb'' \sim \Dcal_{\tilde x}, \\ \alpha', \alpha''\sim \Dcal_{\alpha}}} [\ell_\cont(\xb^{+},\xb^{++},\xb^{-}) ]
\\ &=\EE_y \EE_{\substack{\xb \sim \Dcal_{y}, \\ \bar \xb\sim \Dcal_{\bar y \neq y} }}  \EE_{\substack{\tilde \xb',\tilde \xb'' \sim \Dcal_{\tilde x}, \\ \alpha', \alpha''\sim \Dcal_{\alpha}}} \left[ (1 - \rho_{y})\ell_{\class}\left(\frac{h(\xb^{+} )\T \tilde w}{\zeta(h(\xb^{+} ))}   , y \right) \right]+\EE_y\left[\rho_{y} E\right]
\end{align*}
 Since $\EE_y \EE_{\xb\sim \Dcal_{y}} [\varphi(x)]=\EE_{(\xb,y)\sim \Dcal}[\varphi(x)]=\EE_{\xb\sim \Dcal_{x}}[\varphi(x)]$ given a function $\varphi$ of $x$, we have
\begin{align*}
&  \EE_{\substack{\xb, \bar \xb\sim \Dcal_{x}, \\ \tilde \xb',\tilde \xb'' \sim \Dcal_{\tilde x}, \\ \alpha', \alpha''\sim \Dcal_{\alpha}}} [\ell_\cont(\xb^{+},\xb^{++},\xb^{-}) ]
\\ &= \EE_{\substack{(\xb, y) \sim \Dcal }}   \EE_{\substack{\bar \xb\sim \Dcal_{\bar y}, \\\tilde \xb',\tilde \xb'' \sim \Dcal_{\tilde x}, \\ \alpha', \alpha''\sim \Dcal_{\alpha}}} \left[ \bar \rho({y})\ell_{\class}\left(\frac{h(\xb^{+} )\T \tilde w}{\zeta(h(\xb^{+} ))}   , y \right) \right]+\EE_y[\left(1-\bar \rho({y}) )E\right]
\end{align*}
Taking expectations over $\tilde \xb\sim \Dcal_{\tilde x}$ and $\alpha \sim \Dcal_{\alpha}$ in both sides yields the desired statement.
\end{proof}

\subsection{Proof of Theorem \ref{thm:2}}

We begin by introducing additional notation. Define 
$
\ell_{f, y}(q)=\ell\left(f(q), y \right)  
$
and 
$
\ell_{y}(q)=\ell(q, y) 
$.
Note that
$
\ell\left(f(q), y \right) =\ell_{f, y}(q)=(\ell_{y} \circ f)(q). 
$
The following shows that the contrastive pre-training is related to minimizing the standard classification loss $\ell(f(\xb), y )$ while  regularizing the change of the loss values in the direction of $\delta(\xb,\tilde \xb)$:

\begin{lemma} \label{lemma:6}
Assume that $\ell_{f, y}$ is twice differentiable. Then there exists a function $\varphi$ such that $\lim_{q \rightarrow 0}\varphi(q)=0$ and 
$$
\ell\left(f(\xb^+), y \right) =\ell(f(\xb), y )+\alpha\nabla\ell_{f, y}(\xb)\T\delta(\xb,\tilde \xb)+\frac{\alpha^2}{2}    \delta(\xb,\tilde \xb)\T \nabla^2 \ell_{f,y}(\xb)\delta(\xb,\tilde \xb)+ \alpha^2 \varphi(\alpha).
$$
\end{lemma}
\begin{proof} Let $\xb$ be an arbitrary point in the domain of $f$. Let $\varphi_0(\alpha)=\ell\left(f(\xb^+), y \right)=\ell_{f, y}(\xb+\alpha\delta(\xb,\tilde \xb))$. Then, using the definition of the twice-differentiability of function $\varphi_0$, there exists a function $\varphi$ such that
\begin{align}
\ell\left(f(\xb^+), y \right)=\varphi_0(\alpha)=\varphi_0(0)+\varphi_0'(0)\alpha +\frac{1}{2}  \varphi_0''(0) \alpha^2 + \alpha^2 \varphi(\alpha),
\end{align}
where $\lim_{\alpha \rightarrow 0}\varphi(\alpha)=0$. By chain rule, 
$$
\varphi_0'(\alpha) = \frac{\partial \ell\left(f(\xb^+), y \right) }{\partial\alpha}  = \frac{\partial \ell\left(f(\xb^+), y \right) }{\partial \xb^+} \frac{\partial \xb^+}{\partial\alpha} =\frac{\partial \ell\left(f(\xb^+), y \right) }{\partial \xb^+} \delta(\xb,\tilde \xb)=\nabla\ell_{f, y}(\xb^+)\T\delta(\xb,\tilde \xb)
$$
\begin{align*}
\varphi_0''(\alpha)=\delta(\xb,\tilde \xb)\T \left[ \frac{\partial}{\partial \alpha} \left(\frac{\partial \ell\left(f(\xb^+), y \right) }{\partial \xb^+} \right)\T \right] & =\delta(\xb,\tilde \xb)\T \left[ \frac{\partial}{\partial \xb^+} \left(\frac{\partial \ell\left(f(\xb^+), y \right) }{\partial \xb^+} \right)\T \right] \frac{\partial \xb^+}{\partial\alpha } 
\\ &=\delta(\xb,\tilde \xb)\T \nabla^2 \ell_{f,y}(\xb^+)\delta(\xb,\tilde \xb)
\end{align*}
Therefore,
$$
\varphi_0'(0)=\nabla\ell_{f, y}(\xb)\T\delta(\xb,\tilde \xb)
$$
$$
\varphi_0''(0)=\delta(\xb,\tilde \xb)\T \nabla^2 \ell_{f,y}(\xb)\delta(\xb,\tilde \xb).
$$
By substituting this to the above equation based on  the definition of twice differentiability, 
$$
\ell\left(f(\xb^+), y \right) =\varphi_0(\alpha)=\ell(f(\xb), y )+\alpha\nabla\ell_{f, y}(\xb)\T \delta(\xb,\tilde \xb)+\frac{\alpha^2}{2}    \delta(\xb,\tilde \xb)\T \nabla^2 \ell_{f,y}(\xb)\delta(\xb,\tilde \xb)+ \alpha^2 \varphi(\alpha).
$$
\end{proof}

Whereas the above  lemma is at the level of loss, we now analyze the phenomena at the level of model: 
\begin{lemma} \label{lemma:7}
Let $\xb$ be a fixed  point in the domain of $f$.
Given the fixed $\xb$, let $w \in \mathcal{W}$ be a point such that $\nabla f(\xb) $ and $\nabla^{2} f(\xb)$ exist. Assume that   $f(\xb)=\nabla f(\xb)\T \xb$ and $\nabla^2 f(\xb)=0$. Then we have
\begin{align*}
&\ell\left(f(\xb^+), y \right) 
\\ &=\ell(f(\xb), y )+\alpha(\psi(f(\xb))-y) \nabla f(\xb)\T\delta(\xb,\tilde \xb)+\frac{\alpha^2}{2}     \psi'(f(\xb))|\nabla f(\xb)\T\delta(\xb,\tilde \xb)|^2 + \alpha^2 \varphi(\alpha),
\end{align*}
where $\psi'(\cdot)=\psi(\cdot)(1-\psi(\cdot))>0$.
\end{lemma}
\begin{proof}
Under these conditions, 
$$
\nabla\ell_{f, y}(\xb)=\nabla(\ell_{y} \circ f)(\xb)=\ell_{y}'(f(\xb)) \nabla f(\xb)
$$
$$
\nabla^{2}\ell_{f, y}(\xb)=  \ell_{y}''(f(\xb)) \nabla f(\xb)\nabla f(\xb)\T
+\ell_{y}'(f(\xb))\nabla^{2} f(\xb)= \ell_{y}''(f(\xb)) \nabla f(\xb) \nabla f(\xb)\T
$$
Substituting these into Lemma \ref{lemma:6}  yields
\begin{align*}
&\ell\left(f(\xb^+), y \right)
\\ &=\scalebox{0.9}{$\displaystyle \ell(f(\xb), y )+\alpha\ell_{y}'(f(\xb)) \nabla f(\xb)\T\delta(\xb,\tilde \xb)+\frac{\alpha^2}{2}     \ell_{y}''(f(\xb))\delta(\xb,\tilde \xb)\T [ \nabla f(\xb) \nabla f(\xb)\T]\delta(\xb,\tilde \xb)+ \alpha^2 \varphi(\alpha) $}
\\ &=\ell(f(\xb), y )+\alpha\ell_{y}'(f(\xb)) \nabla f(\xb)\T\delta(\xb,\tilde \xb)+\frac{\alpha^2}{2}     \ell_{y}''(f(\xb))[ \nabla f(\xb)\T\delta(\xb,\tilde \xb)]^2 + \alpha^2 \varphi(\alpha)
\end{align*}
Using Lemma \ref{lemma:1}, we can rewrite this  loss as follows:
\begin{align*}
\ell\left(f(\xb), y \right) = - \log \frac{\exp(yf(\xb))}{1+\exp(f(\xb))}=\log[1+\exp(f(\xb))]-yf(\xb)= \psi_{0}(f(\xb))-yf(\xb)  
\end{align*}
where $\psi_{0}(q)=\log[1+\exp(q)]$.
Thus, 
$$
\ell_{y}'(f(\xb))=\psi_{0}'(f(\xb))-y=\psi(f(\xb))-y
$$
$$
\ell_{y}''(f(\xb)) =\psi_{0}''(f(\xb))=\psi'(f(\xb)) 
$$
Substituting these into the above equation, we have 
\begin{align*}
&\ell\left(f(\xb^+), y \right) 
\\ &=\ell(f(\xb), y )+\alpha(\psi(f(\xb))-y) \nabla f(\xb)\T\delta(\xb,\tilde \xb)+\frac{\alpha^2}{2}     \psi'(f(\xb))[\nabla f(\xb)\T \delta(\xb,\tilde \xb)]^2 + \alpha^2 \varphi(\alpha)
\end{align*}
\end{proof}

The following lemma shows that Mixup version is related to minimize the standard classification loss plus the regularization term on $ \| \nabla f(\xb)\|$.

\begin{lemma} \label{lemma:8}
Let $\delta(\xb,\tilde \xb)= \tilde \xb - \xb$. Let $\xb$ be a fixed  point in the domain of $f$.
 Given the fixed $\xb$, let $w \in \mathcal{W}$ be a point such that $\nabla f(\xb) $ and $\nabla^{2} f(\xb)$ exist. Assume that   $f(\xb)=\nabla f(\xb)\T \xb$ and $\nabla^2 f(\xb)=0$. Assume that  $\EE_{\tilde \xb} [\tilde \xb] =0$. Then, if $y f(\xb) +(y-1)f(\xb)\ge 0$, 
\begin{align*}
&\EE_{\tilde \xb}\ell(f(\xb^+), y ) 
\\ &\ = \ell(f(\xb), y )+c_{1}(\xb)| \| \nabla f(\xb)\|_{2}+c_2(\xb)\| \nabla f(\xb)\|^{2} _{2}+c_3(\xb)\| \nabla f(\xb) \|_{\EE_{\tilde \xb\sim \Dcal_{\tilde x}}[ \tilde \xb \tilde \xb\T]}^2+ O(\alpha^3), \end{align*}
where 
$$
c_1(\xb) = \alpha |\cos( \nabla f(\xb), \xb )|| y-\psi(f(\xb))|  \|\xb\||_{2}\ge 0
$$
$$
c_2(\xb)=\frac{\alpha^2 |\cos( \nabla f(\xb), \xb )|^{2}  \|\xb\||_{2}   }{2}     |\psi'(f(\xb))|\ge 0 
$$
$$
c_3(\xb) =\frac{\alpha^2}{2}     |\psi'(f(\xb))|> 0.
$$
\end{lemma}
\begin{proof}
Using Lemma \ref{lemma:7} with  $\delta(\xb,\tilde \xb)= \tilde \xb - \xb$,
\begin{align*}
&\ell\left(f(\xb^+), y \right) 
\\ &=\scalebox{0.9}{$\displaystyle \ell(f(\xb), y )+\alpha(\psi(f(\xb))-y) \nabla f(\xb)\T( \tilde \xb - \xb)+\frac{\alpha^2}{2}     \psi'(f(\xb))|\nabla f(\xb)\T( \tilde \xb - \xb)|^2 + \alpha^2 \varphi(\alpha) $}
\\ & =\scalebox{0.9}{$\displaystyle \ell(f(\xb), y )-\alpha(\psi(f(\xb))-y) \nabla f(\xb)\T (\xb-\tilde \xb )+\frac{\alpha^2}{2}     \psi'(f(\xb))|\nabla f(\xb)\T (\xb- \tilde \xb)|^2 + \alpha^2 \varphi(\alpha) $}
\\ & =\scalebox{0.9}{$\displaystyle \ell(f(\xb), y )-\alpha(\psi(f(\xb))-y) (f(\xb)- \nabla f(\xb)\T \tilde \xb )+\frac{\alpha^2}{2}     \psi'(f(\xb))|f(\xb)- \nabla f(\xb)\T \tilde \xb|^2+ \alpha^2 \varphi(\alpha) $}
\\ & =\scalebox{0.9}{$\displaystyle \ell(f(\xb), y )+\alpha(y-\psi(f(\xb))) (f(\xb)- \nabla f(\xb)\T \tilde \xb )+\frac{\alpha^2}{2}     \psi'(f(\xb))|f(\xb)- \nabla f(\xb)\T \tilde \xb|^2+ \alpha^2 \varphi(\alpha) $}
\end{align*}
Therefore, using $\EE_{\tilde \xb} \tilde \xb =0$, 
\begin{align*}
&\EE_{\tilde \xb}\ell\left(f(\xb^+), y \right)
\\  & =\ell(f(\xb), y )+\alpha[ y-\psi(f(\xb))]f(\xb)+\frac{\alpha^2}{2}     \psi'(f(\xb))\EE_{\tilde \xb}|f(\xb)- \nabla f(\xb)\T \tilde \xb|^2+ \EE_{\tilde \xb}\alpha^2 \varphi(\alpha)
\end{align*}
Since $|f(\xb)- \nabla f(\xb)\T \tilde \xb|^2=f(\xb)^{2}-2f(\xb) \nabla f(\xb)\T \tilde \xb +(\nabla f(\xb)\T \tilde \xb)^{2}$,
\begin{align*}
\EE_{\tilde \xb}|f(\xb)- \nabla f(\xb)\T \tilde \xb|^2 &=f(\xb)^{2}+\EE_{\tilde \xb}(\nabla f(\xb)\T \tilde \xb)^{2} 
\\ &=f(\xb)^{2}+\nabla f(\xb)\T \EE_{\tilde \xb}[ \tilde \xb \tilde \xb\T]\nabla f(\xb). 
\end{align*} 
Thus, 
\begin{align*}
&\EE_{\tilde \xb}\ell\left(f(\xb^+), y \right) 
\\ & =\scalebox{0.9}{$\displaystyle  \ell(f(\xb), y )+\alpha[ y-\psi(f(\xb))]f(\xb)+\frac{\alpha^2}{2}     |\psi'(f(\xb))|[f(\xb)^2+\nabla f(\xb)\T \EE_{\tilde \xb}[ \tilde \xb \tilde \xb\T]\nabla f(\xb)]+ \EE_{\tilde \xb}\alpha^2 \varphi(\alpha) $}
\end{align*}
The assumption that   $y f(\xb) +(y-1)f(\xb)\ge 0$  implies that $f(\xb)\ge 0$ if $y =1$ and $f(\xb)\le 0$ if $y=0$. Thus, if $y =1$,
$$
[ y-\psi(f(\xb))]f(\xb)=[ 1-\psi(f(\xb))]f(\xb)\ge 0,
$$
since $f(\xb)\ge 0$ and $(1-\psi(f(\xb))) \ge 0$ due to  $\psi(f(\xb)) \in (0,1)$.  If $y =0$,
$$
[ y-\psi(f(\xb))]f(\xb)=-\psi(f(\xb))f(\xb) \ge 0,
$$
since $f(\xb) \le 0$ and $-\psi(f(\xb)) <0$. Therefore, in both cases, 
$$
[ y-\psi(f(\xb))]f(\xb)\ge 0,
$$  
which implies that, \begin{align*}
 y-\psi(f(\xb))]f(\xb) &=[ y-\psi(f(\xb))]f(\xb)  & 
 \\ & =| y-\psi(f(\xb))|| \nabla f(\xb)\T \xb| 
 \\ & =| y-\psi(f(\xb))| \| \nabla f(\xb)\| \|\xb\| |\cos( \nabla f(\xb), \xb )|  
\end{align*}
Therefore,
substituting this and using $f(\xb)= \| \nabla f(\xb)\| \|\xb\| \cos( \nabla f(\xb), \xb )$
\begin{align*}
&\EE_{\tilde \xb}\ell\left(f(\xb^+), y \right) 
\\ & =\scalebox{0.9}{$\displaystyle \ell(f(\xb), y )+c_{1}(\xb) \| \nabla f(\xb) \|_{2}+c_{2}(\xb)\| \nabla f(\xb) \|^{2} _{2}+c_3(\xb) \nabla f(\xb)\T \EE_{\tilde \xb}[ \tilde \xb \tilde \xb\T]\nabla f(\xb)+ \EE_{\tilde \xb}[\alpha^2 \varphi(\alpha)]. $}
\end{align*}
\end{proof}

In the case of Gaussian-noise, we have $\delta(\xb,\tilde \xb)=\tilde \xb \sim \Ncal(0,\sigma^2 I)$:
\begin{lemma} \label{lamma:9}
Let $\delta(\xb,\tilde \xb)=  \tilde \xb \sim \Ncal(0,\sigma^2 I)$. Let $\xb$ be a fixed  point in the domain of $f$.
 Given the fixed $\xb$, let $w \in \mathcal{W}$ be a point such that $\nabla f(\xb) $ and $\nabla^{2} f(\xb)$ exist. Assume that   $f(\xb)=\nabla f(\xb)\T \xb$ and $\nabla^2 f(\xb)=0$.  Then 
$$
\EE_{\tilde \xb \sim \Ncal(0,\sigma^2 I)}\ell\left(f(\xb^+), y \right) =\ell(f(\xb), y ) +\sigma^2c_{3}(\xb) \| \nabla f(\xb)\|_{2}^{2}+ \alpha^2 \varphi(\alpha)
$$
where
$$
c_3(\xb) =\frac{\alpha^2}{2}     |\psi'(f(\xb))|> 0.
$$
\end{lemma}
\begin{proof}
With $\delta(\xb,\tilde \xb)=\tilde \xb \sim \Ncal(0,\sigma^2 I)$, Lemma \ref{lemma:7} yields
\begin{align*}
&\ell\left(f(\xb^+), y \right)
\\ &=\ell(f(\xb), y )+\alpha(\psi(f(\xb))-y) \nabla f(\xb)\T\tilde \xb+\frac{\alpha^2}{2}     \psi'(f(\xb))|\nabla f(\xb)\T\tilde \xb |^2 + \alpha^2 \varphi(\alpha),
\end{align*}
Thus,
\begin{align*}
&\EE_{\tilde \xb \sim \Ncal(0,\sigma^2 I)}\ell\left(f(\xb^+), y \right) 
\\ &=\ell(f(\xb), y ) +\frac{\alpha^2}{2}     \psi'(f(\xb))\EE_{\tilde \xb \sim \Ncal(0,\sigma^2 I)}|\nabla f(\xb)\T \tilde \xb |^2 + \alpha^2 \varphi(\alpha)
    \\ & =\ell(f(\xb), y ) +\frac{\alpha^2}{2}     \psi'(f(\xb))\nabla f(\xb)\T\EE_{\tilde \xb \sim \Ncal(0,\sigma^2 I)}[ \tilde \xb \tilde \xb\T ]\nabla f(\xb)+ \alpha^2 \varphi(\alpha)
\\ & =\ell(f(\xb), y ) +\frac{\alpha^2}{2}     \psi'(f(\xb)) \| \nabla f(\xb)\|_{\EE_{\tilde \xb \sim \Ncal(0,\sigma^2 I)}[ \tilde \xb \tilde \xb\T ]}^{2}+ \alpha^2 \varphi(\alpha)     
\end{align*}
By noticing that 
  $
  \|w\|_{\EE_{\tilde \xb \sim \Ncal(0,\sigma^2 I)}[ \tilde \xb \tilde \xb\T ]}^{2}=\sigma^2 w\T Iw=\sigma^2  \|w\|^2_2
$, this implies the desired statement.
\end{proof}

Combining Lemmas \ref{lemma:8}--\ref{lamma:9} yield the statement of Theorem \ref{thm:2}.

\subsection{Proof of Theorem \ref{thm:4}  }

\begin{proof}
Applying the standard result \citep{bartlett2002rademacher} yields that with probability at least $1-\delta$ ,
$$
\EE_{(\xb,y)}[\one{(2y-1) \neq \sgn(f(\xb)) }]
 -\frac{1}{n} \sum_{i=1}^n \phi((2y_{i}-1)f(\xb_{i}))\le  4L_{\phi} \Rcal_{n}(\Fcal_b^{(\mix)})+  \sqrt{\frac{\ln(2/\delta)}{2n}}.
$$
The rest of the proof bounds the Rademacher complexity $\Rcal_{n}(\Fcal_b^{(\mix)})$.
\begingroup
\begin{align*}
\hat\Rcal_{n}(\Fcal_b^{(\mix)}) &= \EE_{\xi} \sup_{f \in\Fcal_b} \frac{1}{n} \sum_{i=1}^n \xi_i f(\xb_{i})
\\ & = \EE_{\xi} \sup_{w:\|w\|_{\EE_{\tilde \xb \sim \Dcal_{\xb}}[\tilde \xb \tilde \xb\T]}^2 \le b} \frac{1}{n} \sum_{i=1}^n \xi_i w\T \xb_i
\\ & = \EE_{\xi} \sup_{w:w \T \Sigma_X w  \le b} \frac{1}{n} \sum_{i=1}^n \xi_i (\Sigma_X^{1/2}w)\T \Sigma_X^{\dagger/2} \xb_i 
\\ & \le \frac{1}{n}\EE_{\xi} \sup_{w:w \T \Sigma_X w  \le b}  \|\Sigma_X^{1/2}w\|_2  \left\|\sum_{i=1}^n \xi_i \Sigma_X^{\dagger/2} \xb_i \right\|_2 
 \\ & \le \frac{\sqrt b}{n}\EE_{\xi} \sqrt{\sum_{i=1}^n \sum_{j=1}^n \xi_i \xi_j (\Sigma_X^{\dagger/2} \xb_i)\T (\Sigma_X^{\dagger/2} \xb_j) } 
 \\ & \le \frac{\sqrt b}{n} \sqrt{\EE_{\xi}\sum_{i=1}^n \sum_{j=1}^n \xi_i \xi_j (\Sigma_X^{\dagger/2} \xb_i)\T (\Sigma_X^{\dagger/2} \xb_j) }
 \\ & = \frac{\sqrt b}{n} \sqrt{\sum_{i=1}^n  (\Sigma_X^{\dagger/2} \xb_i)\T (\Sigma_X^{\dagger/2} \xb_i) }
 \\ & = \frac{\sqrt b}{n} \sqrt{\sum_{i=1}^n  \xb_i\T \Sigma_X^{\dagger} \xb_i}
\end{align*}
Therefore,
\begin{align*}
\Rcal_{n}(\Fcal_b^{(\mix)}) = \EE_S \hat\Rcal_{n}(\Fcal_b^{(\mix)}) &= \EE_S \frac{\sqrt b}{n} \sqrt{\sum_{i=1}^n  \xb_i\T \Sigma_X^{\dagger} \xb_i} 
\\ & \le \frac{\sqrt b}{n} \sqrt{\sum_{i=1}^n  \EE_{\xb_i} \xb_i\T \Sigma_X^{\dagger} \xb_i} 
\\ & = \frac{\sqrt b}{n} \sqrt{\sum_{i=1}^n  \EE_{\xb_i}  \sum_{k,l} (\Sigma_X^{\dagger})_{kl} (\xb_i)_k(\xb_i)_l} 
\\ & = \frac{\sqrt b}{n} \sqrt{\sum_{i=1}^n    \sum_{k,l}  (\Sigma_X^{\dagger})_{kl} \EE_{\xb_i}(\xb_i)_k(\xb_i)_l}
\\ & = \frac{\sqrt b}{n} \sqrt{\sum_{i=1}^n    \sum_{k,l}  (\Sigma_X^{\dagger})_{kl} (\Sigma_X)_{kl}} 
\\ & = \frac{\sqrt b}{n} \sqrt{\sum_{i=1}^n    \tr(\Sigma_X\T \Sigma_X^{\dagger})} \\ & = \frac{\sqrt b}{n} \sqrt{\sum_{i=1}^n    \tr(\Sigma_X \Sigma_X^{\dagger})} \\ & = \frac{\sqrt b}{n} \sqrt{\sum_{i=1}^n    \rank(\Sigma_X)}
\\ & \le \frac{\sqrt b \sqrt{\rank(\Sigma_X)}}{\sqrt n}
\end{align*}
\endgroup
\allowdisplaybreaks[0]
\end{proof}

\section{Best Hyperparameter Values for Various Experiments}
In general, we found that our method works well for a large range of $\alpha$ values ($\alpha \in [0.6,0.9]$) and $rho$ values ($\rho \in [0.1,0.5]$). In Table \ref{tab:fcn_hyper}, \ref{tab::cnn_cifar10_100_hyper} and \ref{tab::imagenet_hyper}, we present the best hyperparameter values for the experiments in Section \ref{sec:experiments}.
\label{app:hyper}

\begin{table*}[h]
\centering
\begin{tabular}{l l l }
\toprule
Method & Fashion-MNIST & CIFAR10  \\ \midrule
Gaussian-noise & Gausssian-mean=0.1, $\tau$=1.0   & Gausssian-mean=0.05, $\tau$=1.0 \\ 
\dacl &  $\alpha$=0.9, $\tau$=1.0  & $\alpha$=0.9, $\tau$=1.0 \\ 
\daclp & $\alpha$=0.6, $\tau$=1, $\rho$=0.1  & $\alpha$=0.7, $\tau$=1.0, $\rho$=0.5  \\ 
\bottomrule
\end{tabular}
\caption{Best hyperparamter values for experiments on Tabular data (Table \ref{tab:fcn})}
\label{tab:fcn_hyper}
\end{table*}

\begin{table*}
\centering
\begin{tabular}{l l l }
\toprule
Method	& CIFAR10 & CIFAR100  \\ 
\midrule
Gaussian-noise & Gaussian-mean=0.05, $\tau$=0.1  & Gaussian-mean=0.05, $\tau$=0.1 \\
\dacl  &  $\alpha$=0.9,$\tau$=1.0 & $\alpha$=0.9, $\tau$=1.0 \\
\daclp & $\alpha$=0.9, $\rho$=0.1, $\tau$=1.0 & $\alpha$=0.9, $\rho$=0.5, $\tau$=1.0      \\ 
SimCLR & $\tau$=0.5 & $\tau$=0.5 \\ 
SimCLR+\dacl &  $\alpha$=0.7, $\tau$=1.0 & $\alpha$=0.7, $\tau$=1.0 \\ 
\bottomrule
\end{tabular}
\caption{Best  hyperparameter values for experiment of CIFAR10/100 dataset (Table \ref{tab::cnn_cifar10_100})}
\label{tab::cnn_cifar10_100_hyper}
\end{table*}

\begin{table*}
\centering
\begin{tabular}{ll}
\toprule
Method         &  ImageNet  \\ \midrule
Gaussian-noise    & Gaussian-mean=0.1, $\tau$=1.0     \\
\dacl    & $\alpha$=0.9, $\tau$=1.0  \\ 
SimCLR    &  $\tau$=0.1   \\
SimCLR+\dacl   & $\alpha$=0.9, $\tau$=0.1     \\
\bottomrule
\end{tabular}
\caption{Best hyperparameter values for experiments on ImageNet data (Table \ref{tab::imagenet}) }
\label{tab::imagenet_hyper}
\end{table*}

\end{document}